\documentclass[twoside,leqno,twocolumn]{article}

\usepackage[letterpaper]{geometry}

\usepackage{ltexpprt}
\usepackage{hyperref}

\usepackage{cite}
\usepackage{amsmath,amssymb,amsfonts}
\usepackage{algorithmic}
\usepackage{graphicx}
\usepackage{textcomp}
\usepackage{xcolor}
\usepackage{subfigure}
\usepackage{booktabs}
\usepackage{url}
\usepackage{nicefrac}
\usepackage{flushend}

\usepackage[textsize=small]{todonotes}

\newcommand{\specialcell}[2][c]{%
  \begin{tabular}[#1]{@{}c@{}}#2\end{tabular}}

\DeclareMathOperator{\DAO}{\textrm{DAO}}
\DeclareMathOperator{\LOF}{\textrm{LOF}}
\DeclareMathOperator{\ALDR}{\textrm{ALDR}}

\DeclareMathOperator{\MLE}{\textrm{MLE}}

\DeclareMathOperator{\LIDL}{\textrm{LIDL}}
\DeclareMathOperator{\lrd}{\textrm{lrd}}
\DeclareMathOperator{\slrd}{\textrm{slrd}}
\DeclareMathOperator{\SLOF}{\textrm{SLOF}}
\DeclareMathOperator{\knn}{\mathit{k}\textrm{NN}}
\DeclareMathOperator{\nn}{\textrm{NN}}
\DeclareMathOperator{\kdist}{\mathit{k}\textrm{\_dist}}
\DeclareMathOperator{\qq}{\mathbf{q}}
\DeclareMathOperator{\oo}{\mathbf{o}}
\DeclareMathOperator{\pp}{\mathbf{p}}
\DeclareMathOperator{\s}{\mathbf{s}}

\DeclareMathOperator{\reach}{\textrm{reach\_dist}}
\newcommand{\IDop}{\operatorname{ID}}
\newcommand{\ID}{\IDop}
\newcommand{\IDstar}{\IDop^{*}}

\newcommand{\IDesti}[1]{\widehat{\IDop^*_{#1}}}
\newcommand{\expauxOp}{A}
\newcommand{\expaux}[1]{\ensuremath{{\expauxOp}_{#1}}}

\newcommand{\expect}{\mathop{\mathbb{E}}}

\newcommand{\IntrDim}{\mathrm{IntrDim}}

\newcommand{\real}{\mathbb{R}}

\newcommand*\samethanks[1][\value{footnote}]{\footnotemark[#1]}

\def\BibTeX{{\rm B\kern-.05em{\sc i\kern-.025em b}\kern-.08em
    T\kern-.1667em\lower.7ex\hbox{E}\kern-.125emX}}

\begin{document}

\newcommand\relatedversion{}
\renewcommand\relatedversion{\thanks{The full version of the paper can be accessed at \protect\url{https://arxiv.org/abs/1902.09310}}} % Replace URL with link to full paper or comment out this line

%\setcounter{chapter}{2} % If you are doing your chapter as chapter one,
%\setcounter{section}{3} % comment these two lines out.

%\title{\Large Dimensionality-Aware Outlier Detection\relatedversion}
\title{\Large Dimensionality-Aware Outlier Detection: \\
Theoretical and Experimental Analysis}
\author{
Alastair Anderberg
\thanks{The University of Newcastle, Callaghan, NSW, Australia. \texttt{anderberg.alastair@gmail.com}}
\and
James Bailey
\thanks{The University of Melbourne, Parkville, VIC, Australia. \texttt{baileyj@unimelb.edu.au}}
\and
Ricardo J. G. B. Campello
\thanks{University of Southern Denmark, Odense, Denmark. \texttt{\{campello, oli, zimek\}@imada.sdu.dk}}
\and
Michael E. Houle
\thanks{New Jersey Institute of Technology, Newark, NJ, USA. \texttt{michael.houle@njit.edu}}
\samethanks[2]
\and
Henrique O. Marques
\samethanks[3]
\and
Milo\v{s} Radovanovi\'{c}
\thanks{University of Novi Sad, Serbia. %\\ 
\texttt{radacha@dmi.uns.ac.rs}}
\and
Arthur Zimek
\samethanks[3]
}

%\IEEEauthorblockN{Alastair Anderberg$^{1}$}
%\and
%\IEEEauthorblockN{James Bailey$^{2}$}
%\and
%\IEEEauthorblockN{Ricardo J. G. B. Campello$^{3}$}
%\linebreakand
%\IEEEauthorblockN{Michael E. Houle$^{4,2}$}
%\and
%\IEEEauthorblockN{Henrique O. Marques$^{3}$}
%\and
%\IEEEauthorblockN{Milo\v{s} Radovanovi\'{c}$^{5}$}
%\and
%\IEEEauthorblockN{Arthur Zimek$^{3}$}
%\linebreakand
%\IEEEauthorblockA{$^1$\textit{School of Mathematical and Physical Sciences} \\
%\textit{The University of Newcastle}\\
%Callaghan, NSW, Australia \\
%anderberg.alastair@gmail.com}
%\and
%\IEEEauthorblockA{$^2$\textit{Faculty of Engineering and Information Technology} \\
%\textit{The University of Melbourne}\\
%Parkville, VIC, Australia \\
%\{baileyj, mhoule\}@unimelb.edu.au}
%\linebreakand
%\IEEEauthorblockA{$^3$\textit{Dept.\ of Mathematics and Computer Science} \\
%\textit{University of Southern Denmark}\\
%Odense, Denmark \\
%\{campello, oli, zimek\}@imada.sdu.dk}
%\and
%\IEEEauthorblockA{$^4$\textit{Ying Wu College of Computing} \\
%\textit{New Jersey Institute of Technology}\\
%Newark, NJ, USA \\
%michael.houle@njit.edu}
%\and
%\IEEEauthorblockA{$^5$\textit{Faculty of Sciences} \\
%\textit{University of Novi Sad}\\
%Novi Sad, Serbia \\
%radacha@dmi.uns.ac.rs}

\date{}

\maketitle

% Copyright Statement
% When submitting your final paper to a SIAM proceedings, it is requested that you include
% the appropriate copyright in the footer of the paper.  The copyright added should be
% consistent with the copyright selected on the copyright form submitted with the paper.
% Please note that "20XX" should be changed to the year of the meeting.

% Default Copyright Statement
%\fancyfoot[R]{\scriptsize{Copyright \textcopyright\ 20XX by SIAM\\
%Unauthorized reproduction of this article is prohibited}}

% Depending on which copyright you agree to when you sign the copyright form, the copyright
% can be changed to one of the following after commenting out the default copyright statement
% above.

%\fancyfoot[R]{\scriptsize{Copyright \textcopyright\ 20XX\\
%Copyright for this paper is retained by authors}}

%\fancyfoot[R]{\scriptsize{Copyright \textcopyright\ 20XX\\
%Copyright retained by principal author's organization}}

%\pagenumbering{arabic}
%\setcounter{page}{1}%Leave this line commented out.

%===========================================
% \bigskip\hrule\bigskip %====================
%===========================================

\begin{abstract}
\small\baselineskip=9pt
We present a nonparametric method for outlier detection that takes full account of local variations in intrinsic dimensionality within the dataset. Using the theory of Local Intrinsic Dimensionality (LID), our `dimensionality-aware' outlier detection method, $\DAO$, is derived as an estimator of an asymptotic local expected density ratio involving the query point and a close neighbor drawn at random. The dimensionality-aware behavior of $\DAO$ is due to its use of local estimation of LID values in a theoretically-justified way.
Through comprehensive experimentation on more than 800 synthetic and real datasets, we show that $\DAO$ significantly outperforms three popular and important benchmark outlier detection methods: Local Outlier Factor (LOF), Simplified LOF, and $\knn$.
\end{abstract}

%\begin{keywords}
%Outlier Detection, Intrinsic Dimensionality
%\end{keywords}

\section{Introduction}
Outlier detection, one of the most fundamental tasks in data mining, aims to identify observations that deviate from the general distribution of the data. Such observations often deserve special attention as they may reveal phenomena of extreme importance, such as 
%traffic accidents~\cite{djenouri2018}, 
network intrusions~\cite{ahmad2021}, 
%credit card fraud~\cite{adewumi2017}, 
sensor failures~\cite{ramotsoela2018}, or disease~\cite{alaverdyan2020}.

The study of outliers has its origins in the field of statistics. There exist dozens of parametric statistical tests 
%based on a sound theory 
that can be applied to detect outliers~\cite{Bar78,Haw80}. Although these tests have shown good performance when the underlying theoretical assumptions are met, in real-world applications these assumptions usually do not hold~\cite{marques2023}. This well-known limitation of the parametric approach has triggered research on unsupervised, nonparametric methods for outlier detection, as far back as the seminal work of Knorr and Ng in 1997~\cite{KnoNg97a}. Nonparametric approaches make no explicit assumptions on the nature of the underlying data distribution, but can estimate some of its local characteristics, such as probability density at a point of interest. Although nonparametric methods are usually more suitable for real-world applications due to their flexibility, generally speaking they lack the theoretical justification that parametric approaches have enjoyed~\cite{ZimekF18}.

The estimates computed by non-parametric methods for outlier detection usually rely on the distances from the test point to its nearest neighbors.
As such, these methods are subject to the well-known `curse of dimensionality' phenomenon, by which the quality of distance information diminishes as the dimensionality of the data increases~\cite{ZimSchKri12}, leading to such observable effects as the concentration of distance values about their %mean~\cite{BeyGolRamSha99,Pes00,FraWerVer07}. 
mean~\cite{BeyGolRamSha99}. 
Contrary to what is commonly assumed, however, most of the challenges associated with high-dimensional data analysis do not depend directly on the representational data dimension (number of attributes); rather, they are better explained by the notion of `intrinsic dimensionality' (ID), which can be understood intuitively as the number of features required to explain the distributional characteristics observed within the data, or the dimension of the surface (or manifold or subspace) achieving the best fit to the data. In practice, for many models of ID, the estimated number of explanatory features or surface dimensions is often much smaller than the dimension of the embedding data space.

%The notion of ID comes from the fact that most of the variation in real data may be explained by fewer features than the dimensionality of the embedding data space. The dimension of the surface which best approximates the data is what is called ID. 

Typically, the intrinsic dimension is not uniform across the whole dataset: applying a model of ID to subregions of the data (such as the neighborhood of a query point) almost always produces different results for each.
%However, the ID is typically not uniform across the whole dataset. On the contrary, the ID is often varying considerably for different regions of a dataset and has therefore been characterized locally, providing for each region (e.g., around a query point) of the dataset a different Local Intrinsic Dimensionality (LID)~\cite{Hou13,Hou17a}. 
%Variation of ID within the dataset increases even further the complexity of the data. This includes not only the range of the variability but also the complexity of the intrinsic dimension profile within a dataset that often results in a performance loss for data mining and indexing methods~\cite{KarRuh02,HouKriKroSchetal10,ZimSchKri12,MaLiWanErfetal18,AumCec21,AmsalegBBEFHRN21}.
The added complexity associated with variation of local ID has the potential to increase the difficulty of data analysis, thereby resulting in a performance loss for many similarity-based data mining and indexing %methods~\cite{KarRuh02,HouKriKroSchetal10,ZimSchKri12,MaLiWanErfetal18,AumCec21,AmsalegBBEFHRN21}. 
methods~\cite{KarRuh02,HouKriKroSchetal10,ZimSchKri12,AumCec21,AmsalegBBEFHRN21}. 
For this reason, there has been considerable recent attention to `local' models of ID as an alternative to the classic `global' models that seek to characterize the complexity of the entire dataset. 

In this paper, we focus on a theoretical model of local complexity, the Local Intrinsic Dimensionality (LID)~\cite{Hou13,Hou17a}, 
which was originally motivated by the need to characterize the interrelationships between probability and distance within neighborhoods.
%, in terms of the number of `effective' features (the `intrinsic' dimension) that would best explain the observations. 
Unlike other measures of ID which are formulated as heuristics for direct use on discrete datasets, LID is a theoretical quantity for which effective estimators have been %developed~\cite{AmsalegCFGHKN18,amsaleg2022intrinsic,Hil75,LevBic04}.
developed~\cite{AmsalegCFGHKN18,amsaleg2022intrinsic,LevBic04}.
LID has had many recent successes in data analysis, both theoretical and practical, in areas such as search 
and indexing~\cite{AumCec21,CasEngHouKroetal17},
AI and data mining~\cite{BaiHouMa22,RomCheNguBaietal16}, and deep %learning~\cite{AmsalegBBEFHRN21,LiQiZhaetal19,MaLiWanErfetal18,MaWanHouZhoetal18}. 
learning~\cite{AmsalegBBEFHRN21,MaWanHouZhoetal18}. 
There is empirical evidence to suggest that outlierness is correlated in practice with high local intrinsic dimensionality~\cite{HouSchZim18}.

The main contribution of our paper is the first known nonparametric method for \emph{dimensionality-aware outlier detection} ($\DAO$), one whose formulation we derive as an estimator of an \emph{asymptotic local expected density ratio} ($\ALDR$), using the theory of LID. The dimensionality-aware behavior of $\DAO$ is due to its use of local estimation of LID values in a theoretically-justified way.
%In contrast to the traditional outlier detection methods, we propose a model for \emph{Dimensionality-Aware Outlier Detection} ($\DAO$). 
Our proposed model will be seen to resemble the classic LOF outlier detection criterion~\cite{BreKriNgSan00}, and (even more closely) its very popular simplified variant, SLOF~\cite{SchZimKri14}. However, like all known nonparametric outlier detection criteria, LOF and SLOF both differ from our proposed model in that they rely solely on distance-based criteria for density estimation, without taking local dimensionality explicitly into account. 

Through our theoretical model we gain an understanding of the susceptibility of traditional outlier detection methods to variation in local ID within the dataset. As a second main contribution, we verify this understanding through a comprehensive empirical study (involving a total of more than 800 datasets) of the performance of $\DAO$ versus three of the most popular and effective nonparametric outlier detection methods known to date:  $\LOF$, $\SLOF$, and 
%$\knn$~\cite{AngPiz02,RamRasShi00}. 
$\knn$~\cite{RamRasShi00}. 
In particular, we present visualizations of outlier detection performance for 393 real datasets, that empirically confirm the tendency of $\DAO$ to outperform its dimensionality-unaware competitors, particularly when the variation of LID values within a dataset is high (as indicated by measures of high dispersion or low autocorrelation).

%and different from most approaches in the literature that aim to mitigate problems associated with deterioration of the quality of the distances, we can provide a comprehensive and theoretically well-justified solution to handle the problem.

% In summary, we make the following contributions in this paper:
% \begin{itemize}
%     \item 
% \end{itemize}

In this paper, we provide the full theoretical and experimental details for research first published at the SIAM Conference on Data Mining (SDM24)~\cite{DAO-SDM24}. In Section~\ref{sec:related_work} we discuss related work. In Section~\ref{sec:background} we provide the reader with the relevant background for the LOF and SLOF outlier detection methods, and the theory of local intrinsic dimensionality. In Section~\ref{sec:dao} we derive and theoretically justify our proposed dimensionality-aware outlierness model, $\DAO$. We present our experimental setup in Section~\ref{sec:eval}, and discuss the results in Section~\ref{sec:results}. Finally, we present concluding remarks in Section~\ref{sec:conclusions}.

\section{Related Work}
\label{sec:related_work}

Non-parametric approaches for outlier detection~\cite{CamZimSanCametal16} either explicitly or implicitly aim to assess the density in the vicinity of a query point, such that points with the lowest densities are reported as the strongest outlier candidates.
The assessment of density can be direct, or based on distances, or on the ratio of one density with respect to another.
The distance-based DB-outlier method~\cite{KnoNg97a} estimates density in the vicinity of a query by counting the number of data points contained in a neighborhood with predefined radius.
Conversely, the $k$-nearest-neighbor algorithm
%($\knn$)~\cite{AngPiz02,RamRasShi00} 
($\knn$)~\cite{RamRasShi00} 
measures the radius needed so as to capture a fixed number of points, $k$.
%within the neighborhood of the point by computing the distance of a point $\qq$ to its $k^{th}$ nearest neighbor, i.e., $\knn(\qq) = \kdist(\qq)$.
Local outlier detection methods based on density ratios, such as LOF~\cite{BreKriNgSan00}, %and SLOF~\cite{SchZimKri14} 
identify outliers to be those points having local densities that are small relative to those of their nearest neighbors.

%In contrast, global outlier detection methods compare and rank the density estimate of a query point with respect to the density estimates of all other data points --- here, the entire dataset is used as a reference set for the determination of outliers. As with local outlier detection, global outlier detection can be distance-based or density based. 
%Density-based methods also estimate the density around points based on the $k$ nearest neighbor distances. However, local density-based methods such as LOF \cite{BreKriNgSan00} and SLOF \cite{SchZimKri14} search for local outliers, i.e., points that are located in an area of relative low density when compared to the density around the points in their neighborhood.

Many variations of the aforementioned outlier models have been proposed over the past decades. Some rely on nonstandard notions of neighborhood, such as COF (connectivity-based outlier factor)~\cite{TanCheFuChe02}, INFLO (Influenced Outlierness)~\cite{JinTunHanWan06}, and others based on %reverse nearest neighbors~\cite{HauKaeFra04,RadNanIva14}.
reverse nearest neighbors~\cite{RadNanIva14}.
Others estimate the local density in different ways, such as LDF (Local Density Factor)~\cite{LatLazPok07}, 
LOCI (Local Outlier Integral)~\cite{PapKitGibFal03}, 
and 
KDEOS (Kernel Density Estimation Outlier
Score)~\cite{SchZimKri14a}. 
Yet other variations derive an outlier score from a comparison of a local model for the query point to local models of other data points; these include the local distance-based outlier detection (LDOF) approach~\cite{ZhaHutJin09}, 
probabilistic modeling of local outlier scores (LoOP)~\cite{KriKroSchZim09a}, or meta-modeling of outlierness~\cite{KriKroSchZim11}.
A somewhat different approach is angle-based outlier detection (ABOD)~\cite{KriSchZim08}, which bases the degree of outlierness of a query point on the variance of the angles formed by it and other pairs of points. %However, this method also takes the distances into account as weights.

Despite the many variants of the fundamental techniques of non-parametric outlier detection that have appeared over the past quarter century, two classic methods in particular, $\LOF$ and $\knn$, have repeatedly been confirmed as top performers or recommended baselines in larger comparative studies involving local anomaly detection~\cite{CamZimSanCametal16,goldstein2016,HanHuHuaJiaZha22}. None of these methods, however, take into account the possibility of variation in local intrinsic dimensionality within the dataset.

\section{Background}\label{sec:background}

\subsection{Local Outlier Factor.}
The term `local outlier' refers to an observation that is sufficiently different from observations in its vicinity. 
%This can be characterized using neighborhood information. The neighborhood of a point $q$ in $\mathbb{R}^d$ can be considered as the ball of a given radius $r$ centered on $q$, $B(q,r)$. A dataset can be considered a finite sample from a continuous underlying distribution, where we expect a larger number of observations to be located in areas of higher probability mass in the underlying distribution, while the outliers are located in areas of lower probability mass. 
%
Typical density-based outlier detection methods consider a point $\qq$ as a local outlier if a given neighborhood of $\qq$ is less dense than neighborhoods centered at $\qq$'s own neighbors, according to some criterion.
Following this principle, Local Outlier Factor (LOF)~\cite{BreKriNgSan00} contrasts the local density at $\qq$ with the local densities at the members of its $k$-nearest neighbor set, $\nn_k(\qq)$:
\begin{equation*}
\label{lof}
    \LOF_k(\qq)
    \:\:\triangleq\:\:
    \frac{1}{k}\sum_{\oo\in \nn_k(\qq)}\frac{\lrd_k(\oo)}{\lrd_k(\qq)},
\end{equation*}
where the local reachability density ($\lrd$) at point $\pp$ is defined in terms of the inverse of an average of so-called `reachability distances' taken from the $k$-nearest neighbors of $\pp$:
%Since density is equal to $mass/volume$, in our $\knn$-based estimate, the mass is the number of points in the $k$NN and the volume is a measure indicative of the distance to each point. It is given as 
\begin{equation*}
\lrd_k(\pp)    
\:\:\triangleq\:\:
\left(\frac{\sum_{\s\in \nn_k(\pp)}\reach_k(\pp\,{\leftarrow}\s)}{k}\right)^{-1}.
\label{eq:lrdlof}
\end{equation*}
Such a distance is defined as the maximum of the neighbor's own $k$-NN distance, $\kdist(\s)$, and its distance to $\pp$, $d(\pp,\s)$:
\begin{equation*}\label{reach-dist}
    \reach_k(\pp\,{\leftarrow}\s) = \max\{\kdist(\s), d(\pp,\s) \}
    \, .
\end{equation*}
The LOF reachability distance can be regarded as using the distance between $\pp$ and its neighbor $\s$ by default, except when $\s$ is closer to $\pp$ than it is to its own $k$-th nearest neighbor. 

Although the local reachability density of LOF aggregates the contribution of many neighbors to produce a smoother and more stable estimate, it requires multiple levels of neighborhood computation (for each neighbor $\oo$ of $\qq$, the $k$-NN distance of each of the neighbors of $\oo$).  
The Simplified LOF (SLOF) variant~\cite{SchZimKri14} avoids one level of neighborhood computation by using the inverse $k$-NN distance in place of the local reachability density:
\begin{equation*}
    \label{eqn:SLOF1}
    \SLOF_{k}(\qq)
    \:\:\triangleq\:\:
    \frac{1}{k}\sum_{\oo\in \nn_k(\qq)}\frac{\slrd_k(\oo)}{\slrd_k(\qq)}
    \, ,
\end{equation*}
where
\begin{equation*}\label{eq:slrd}
    \slrd_k(\pp)
    \:\:\triangleq\:\:
    \frac{1}{\kdist(\pp)}
    \, .
\end{equation*}

The density of the neighborhood $\nn_k(\qq)$ can be regarded as the ratio between the mass (the number of points $k$) and the volume of the ball with radius $\kdist(\qq)$. In the Euclidean setting, this ratio is proportional to $\nicefrac{k}{(\kdist(\qq))^m}$, where $m$ is the dimension of the space. $\slrd_k(\qq)$ can thus be interpreted as a proportional density estimate that treats $k$ as a constant, and ignores the dimension of the ambient space, $m$.

\subsection{Local Intrinsic Dimensionality.}

%The Local Intrinsic Dimensionality (LID) model~\cite{Hou13,Hou17a} is a model of complexity originally motivated to characterize the interrelationships between probability and distance in neighborhoods, in terms of the number of `effective' features (the `intrinsic' dimension) that would best explain the observations. LID has had many recent successes, both theoretical and practical, in such application areas as search and indexing~\cite{AumCec21,CasEngHouKroetal17}, AI and data mining~\cite{BaiHouMa22,RomCheNguBaietal16}, and deep learning~\cite{AmsalegBBEFHRN21,LiQiZhaetal19,MaLiWanErfetal18,MaWanHouZhoetal18}. There is empirical evidence to suggest that outlierness is correlated in practice with high local intrinsic dimensionality~\cite{HouSchZim18}.

The Local Intrinsic Dimensionality (LID) %model~\cite{Hou13,Hou17a} 
model~\cite{Hou17a} 
can be regarded as a continuous extension of the expansion dimension due to Karger and %Ruhl~\cite{HouKasNet12,KarRuh02}, 
Ruhl~\cite{KarRuh02}, 
which derives
a measure of dimensionality from the relationship between volume and radius in an expanding ball centered at a point of interest in a Euclidean data domain. Given two measurements of radii ($r_1$ and $r_2$) and volume ($V_1$ and $V_2$), the dimension $m$ can be obtained from the ratios of the measurements:
\[
\frac{V_2}{V_1}
=
\left(
\frac{r_2}{r_1}
\right)^m
\:\:
\Longrightarrow
\:\:\:
m
=
\frac{
\ln(\nicefrac{V_2}{V_1})
}
{
\ln(\nicefrac{r_2}{r_1})
}
\, .
\]

Early expansion models are discrete, in that they estimate volume by the number of data points captured by the ball. The LID model, by contrast, allows data to be viewed as samples drawn from an underlying distribution, with the volume of a ball represented by the probability measure associated with its interior. For balls centered at a common reference point, the probability measure can be expressed as a function $F(r)$ of the radius $r$, and as such $F$ can be viewed as the cumulative distribution function (CDF) of the distribution of distances to samples drawn from the underlying global distribution. However, it should be noted that the LID model has been developed to characterize the complexity of growth functions in general:
%, and is not restricted to any distributional interpretation: 
the variable $r$ need not be a Euclidean distance, and the function $F$ need not satisfy the conditions of a CDF.

\begin{Definition}[\!\!\cite{Hou17a}]
\label{D:IntrDim}
Let $F$ be a real-valued function that is non-zero 
over some open interval containing $r\in\real$, $r\neq 0$. 
The {\em intrinsic dimensionality
of $F$ at $r$} is defined as follows, whenever the limit exists:
%\vspace{-2mm}
\begin{eqnarray*}
\IntrDim_{F}(r)
&\triangleq&
\lim_{\epsilon\to 0}
\frac{\ln \left( F((1{+}\epsilon)r) / F(r)\right)}
{\ln (1{+}\epsilon)}
\, .
%\vspace{-2mm}
%\]
\end{eqnarray*}
\end{Definition}

When $F$ is `smooth' (continuously differentiable) in the vicinity of $r$,
its intrinsic dimensionality has a closed-form expression: 

\begin{theorem}[\!\!\cite{Hou17a}]
\label{T:fundamental}
Let $F$ be a real-valued function that is non-zero 
over some open interval containing $r\in\real$, $r\neq 0$. 
If $F$ is continuously differentiable at $r$, then
\[
\ID_F(r)
\:\:
\triangleq
\:\:
\frac{r\cdot F'(r)}{F(r)}
\:\:
=
\:\:
\IntrDim_{F}(r)
\, .
\]
\end{theorem}

%Let $\mathbf{x}$ be a location of interest within a data domain $\mathcal{S}$ for which the distance measure $d$ has been defined.
%To any generated sample $\mathbf{y}\in\mathcal{D}$ 
%we can associate the distance $r=d(\mathbf{x},\mathbf{y})$; in this 
%way, the global distribution that produces samples $\mathbf{y}$ can be said to induce
%a local distance distribution with CDF $F$ with respect to $\mathbf{x}$.

In characterizing the local intrinsic dimensionality at a query location, we
are interested in the limit of $\ID_F(r)$ as the distance $r$ tends to $0$, which we denote by 
\[
\IDstar_F
\:\:\triangleq\:\:
\lim_{r\to 0}\ID_F(r)
\, .
\]
Henceforth, when we refer to the local intrinsic dimensionality of a function $F$, or of a reference location whose induced distance distribution has $F$ as its CDF, we will take `LID' to mean the quantity $\IDstar_{F}$.

To gain a better intuitive understanding of LID and how it can be interpreted, consider the ideal case in which points in the neighborhood of $\mathbf{x}$ are distributed uniformly within a submanifold in $\mathcal{D}$. Here, in this ideal setting, the dimension of the submanifold would equal $\IDstar_F$. In general, however, data distributions are not ideal, the manifold model of data does not perfectly apply, and $\IDstar_F$ is not necessarily an integer. In practice, estimation of the LID at $\mathbf{x}$ would give an indication of the dimension of the submanifold containing $\mathbf{x}$ that best fits the distribution.

\subsection{LID Representation Theorem.}

The intrinsic dimensionality function $\ID_{F}$ is known to fully characterize its associated function $F$. The LID Representation Theorem~\cite{Hou17a}, which we state below, is analogous to a foundational result from the statistical theory of extreme values (EVT), the Karamata Representation Theorem for regularly varying functions. For more information on EVT, the Karamata representation, and how the LID model relates to it, we refer the reader to~\cite{Coles01,Hou17a,Hou17b}. 

\begin{theorem}[LID Representation~\cite{Hou17a}]
\label{T:id-rep}
Let $F:\real\to\real$ be a real-valued function,
and assume that $\IDstar_{F}$ exists.
Let $r$ and $w$ be values
for which $r/w$ and $F(r)/F(w)$ are both positive.
If $F$ is non-zero and continuously differentiable
everywhere in the interval
$[\min\{r,w\},\max\{r,w\}]$, then
\begin{eqnarray*}
\frac{
F(r)
}{
F(w)
}
& = &
%\, = \,
\left(\frac{r}{w}\right)^{\IDstar_{F}}
\cdot
\expaux{F}(r,w),
\end{eqnarray*}
where
\begin{eqnarray*}
\expaux{F}(r,w)
& \triangleq &
      \exp\left(
         \int_{r}^{w}
             \frac{\IDstar_{F}-\ID_{F}(t)}{t}
         \,\mathrm{d}t
      \right)
,
\end{eqnarray*}
whenever the integral exists.
\end{theorem}

The convergence characteristics of $F$ to its asymptotic form are expressed by the auxiliary factor $\expaux{F}(r,w)$, which is related to the slowly-varying functions studied in EVT~\cite{Coles01}. In~\cite{Hou17a}, $\expaux{F}(r,w)$ is shown to tend to 1 as $r,w\to 0$, provided that the log-ratio $\ln(\nicefrac{r}{w})$ remains bounded. With these restrictions, the auxiliary factor disappears from the statement of Theorem~\ref{T:id-rep} when the limit of both sides is taken. 

When the relationship between $r$ and $w$ is made explicit through a parameterization $w=\alpha(u)$ and $r=\beta(u)$, an asymptotic version of the LID Representation Theorem can be formulated without reference to the auxiliary factor $\expaux{F}$.

\begin{theorem}[\!\!\cite{Houle20}]
\label{T:densityratio}
Let $F:\real_{\geq 0}\to \real_{\geq 0}$ be a non-decreasing function,
and assume that $\IDstar_{F}$ exists and is positive.
Let $\alpha,\beta:\real_{\geq 0}\to\real_{\geq 0}$
be functions such that $\alpha(0)=\beta(0)=0$,
and for some value of $c>0$,
their restrictions to the interval
$[0,c)$ are continuously differentiable and
strictly
monotonically increasing.
Then
\begin{eqnarray}
~~~~~~
\lim_{u\to 0}
\frac{F(\beta(u))}{F(\alpha(u))}
& \: = \: &
\lim_{u\to 0}
\left(
\frac{\beta(u)}{\alpha(u)}
\right)^{\IDstar_{F}}
\:\: = \:\:
\lambda^{\IDstar_{F}}
\label{E:densityratio}
\end{eqnarray}
whenever the limit $\lambda=\lim_{u\to 0} \frac{\beta(u)}{\alpha(u)}$
exists.
If instead $\lambda$ diverges to $+\infty$,
then the limits in Equation~\ref{E:densityratio} both diverge.
\end{theorem}

In the context of a distance distribution and its CDF $F$, Theorem~\ref{T:densityratio} shows the convergence between the ratio of two neighborhood probabilities with different radii, and the ratio of these radii taken to the power of the local intrinsic dimensionality, $\IDstar_F$. In the following section, this result will be used to derive a dimensionality-aware estimator of a distributional local density-based outlierness criterion. 

\section{The Dimensionality-Aware Outlier Model}
\label{sec:dao}

As discussed previously, local outliers can in general be found by comparing the density of the neighborhood of a point to the densities of the neighborhoods of that point's neighbors. Here, we emulate the design choices of traditional (discrete) density-based outlier models using distributional concepts, to produce a theoretical dimensionality-aware model that treats the dataset as samples drawn from some unknown underlying distribution that is assumed to be continuous everywhere except (perhaps) at the query location. After establishing our model, we develop a practical estimator of outlierness suitable for use on discrete datasets.

\subsection{Asymptotic Local Density Ratio.}

For any distribution over an isometric representation space, the volume $V(\epsilon)$ of any ball of radius $\epsilon$ is the same throughout the domain. For a suitably small choice of $\epsilon$, we consider the density at a point $\pp$ to be the probability measure $F_{\pp}(\epsilon)$ associated with its $\epsilon$-neighborhood ball $B_{\pp}(\epsilon)$, divided by the volume of the ball, $V(\epsilon)$. Note that in any such density ratio between a test point $\qq$ and its neighbor $\oo$, the volumes cancel out to produce the simple ratio $\nicefrac{F_{\oo}(\epsilon)}{F_{\qq}(\epsilon)}$. Rather than aggregating density ratios for a fixed number of discrete neighbors, we instead reason in terms of the expectation of density ratios involving a random sample $\oo$ drawn from $B_{\qq}(\epsilon)$:
\begin{equation*}
\label{E:generaloutlier}
\expect_{\oo\in B_{\qq}(\epsilon)}
\left[
\frac{
F_{\oo}(\epsilon)
}{
F_{\qq}(\epsilon)
}
\right]
\, .
\end{equation*}

In practice, models for outlier detection are faced with the problem of deciding the neighborhood radius $\epsilon$, or neighborhood cardinality $k$. Here, we resolve this issue by  examining the tendency of our density-based criterion as the ball radius tends to zero, thereby obtaining a ratio of infinitesimals. Accordingly, we define the asymptotic local expected density ratio (ALDR) of a query point $\qq$ to be:
\begin{eqnarray*}
\ALDR (\qq)
&\triangleq&
\lim_{\epsilon\to 0^{+}}
\expect_{\oo\in B_{\qq}(\epsilon)}
\left[
\frac{
F_{\oo}(\epsilon)
}{
F_{\qq}(\epsilon)
}
\right]
\, .
\end{eqnarray*}

Intuitively, an ALDR score of 1 is associated with inlierness:
it indicates that the probability measure function $F_{\qq}$
in the vicinity
of the test point $\qq$ agrees perfectly with that of its neighbors,
in that their (expected) local probability measures
$F_{\oo}$
converge to 
$F_{\qq}$
as $\oo$ tends to $\qq$.

A limit value different than 1 
indicates a discontinuity of the neighborhood
probability measure 
at $\qq$ relative to its neighbors.
Limit values greater than 1 (including the case where ALDR diverges to infinity) 
are associated with outlierness in the usual sense of sparseness: they
indicate that the test point has a local probability measure 
too small to be consistent with those of its neighbors in the domain. 
Limit values less than 1 can also be regarded as anomalous, in that they
can be interpreted as an abnormally large concentration of probability
measure at the individual point $\qq$. 
In both cases, the degree of discontinuity can be regarded as increasing 
as the ratio diverges from 1.
In this paper, however, we will be concerned with identifying cases for which the $\ALDR$ score exceeds 1 (sparse outlierness).
%Equivalently, limit values different from 1 can also be taken as an indication that the distance measure from $\qq$ to the remaining points of $\domain$ is anomalous in nature.

%(DAO avoids the pitfalls of definitions of outlierness based on the notion of probability density...)

\subsection{Dimensionality-Aware Reformulation of ALDR.}

For the purposes of deriving an estimator of $\ALDR$, we introduce a minor reformulation.
Instead of taking the radius of the ball $B_{\qq}$ to be the same as the neighborhood radius within which probability measure is assessed around $\qq$ and $\oo$, we decouple the rates by which these radii tend to zero. In the reformulation, the inner limit controls the neighborhood radius, and the outer limit controls the ball radius.
\begin{eqnarray*}
\ALDR' (\qq)
&\triangleq&
\lim_{\epsilon\to 0^{+}}
\expect_{\oo\in B_{\qq}(\epsilon)}
\left[
\lim_{\gamma\to 0^{+}}
\frac{
F_{\oo}(\gamma)
}{
F_{\qq}(\gamma)
}
\right]
\, .
\end{eqnarray*}

With this decoupling, we 
apply Theorem~\ref{T:densityratio} to 
convert the ratio of neighborhood probabilities $\nicefrac{F_{\oo}(\gamma)}{F_{\oo}(\gamma)}$ to one that involves only distances and LID values.
Given a probability value $p\in[0,1]$ and any point $\oo$ in the domain,
let $\delta_{\oo}(p)$ be the infimum of the distance values $r$ for which
$F_{\oo}(r)=p$. This definition ensures that if 
$F_{\oo}$ is continuously differentiable at $\delta_{\oo}(p)$,
then $F_{\oo}(\delta_{\oo}(p))=p$, and the distance function $\delta_{\oo}$ is also continuously differentiable at $p$.

\begin{theorem}
\label{T:ALDR}
Let $\qq$ be a query point.
If there exists a constant $c>0$ such that 
for all $\oo\in B_{\qq}(c)\setminus\{\qq\}$,
\begin{itemize}
\item
$F_{\oo}$ is continuously differentiable over the range
$[0,c]$,
\item
$\IDstar_{F_{\oo}}$ exists and is positive,
and
\item
the limit of 
$\nicefrac{
\delta_{\qq}(p)
}{
\delta_{\oo}(p)
}$ 
as $p\to 0^{+}$ either exists or diverges to $+\infty$,
\end{itemize}
then
\begin{eqnarray*}
\ALDR' (\qq)
& = &
\lim_{\epsilon\to 0^{+}}
\expect_{\oo\in B_{\qq}(\epsilon)}
\left[
\lim_{p \to 0^{+}}
\left(
\frac{
\delta_{\qq}(p)
}{
\delta_{\oo}(p)
}
\right)^{\IDstar_{F_{\oo}}}
\right]
.
\end{eqnarray*}
\end{theorem}

\begin{proof}
From the assumption of continuous differentiability
of $F_{\oo}$ for all $\oo\in B_{\qq}(c)$,
we can determine a probability value $p_{0}$
such that $\delta_{\oo}(p_{0})<c$.
This allows the distance variable 
$\gamma$ in the definition of $\ALDR'$ to be
expressed in terms of a neighborhood probability $p<p_0$,
with $\gamma=\delta_{\qq}(p)$. The inner limit
can then be stated as a tendency of $p$ to zero, 
as follows:
\begin{eqnarray*}
\ALDR' (\qq)
&=&
\lim_{\epsilon\to 0^{+}}
\expect_{\oo\in B_{\qq}(\epsilon)}
\left[
\lim_{p\to 0^{+}}
\frac{
F_{\oo}(\delta_{\qq}(p))
}{
F_{\qq}(\delta_{\qq}(p))
}
\right]
\, .
\end{eqnarray*}
Noting that $F_{\qq}(\delta_{\qq}(p))=F_{\oo}(\delta_{\oo}(p))=p$,
we obtain
\begin{eqnarray}
\label{E:ALDR}
~~~~~~~\ALDR' (\qq)
&=&
\lim_{\epsilon\to 0^{+}}
\expect_{\oo\in B_{\qq}(\epsilon)}
\left[
\lim_{p\to 0^{+}}
\frac{
F_{\oo}(\delta_{\qq}(p))
}{
F_{\oo}(\delta_{\oo}(p))
}
\right]
\, .
%\nonumber
\end{eqnarray}

For any
$\oo\in B_{\qq}(r)\setminus\{\qq\}$,
from the assumption that
$F_{\oo}$ is continuously differentiable over the range $[0,r]$
and that $\IDstar_{F_{\oo}}$ exists,
Theorem~\ref{T:densityratio} implies that
\begin{eqnarray*}
\lim_{p \to 0^{+}}
\frac{
F_{\oo}(\delta_{\qq}(p))
}{
F_{\oo}(\delta_{\oo}(p))
}
&=&
\lim_{p \to 0^{+}}
\left(
\frac{
\delta_{\qq}(p)
}{
\delta_{\oo}(p)
}
\right)^{\IDstar_{F_{\oo}}}
%\, .
.
\end{eqnarray*}
Substituting into Equation~\ref{E:ALDR}, the result follows.
\end{proof}

In the proof of Theorem~\ref{T:ALDR}, we took advantage of the formulation of $\ALDR'$, in that by isolating the limit of the density ratio $\nicefrac{F_{\oo}(\gamma)}{F_{\oo}(\gamma)}$, we could apply Theorem~\ref{T:densityratio} to convert it to a form that depends on neighborhood radii and LID values. It should be noted that this conversion cannot not be performed on the original formulation, $\ALDR$, due to the nesting of the density ratio within an expectation taken over a shrinking ball. In general, the order of limits of functions (including integrals and continuous expectation) cannot be interchanged unless certain conditions are known to hold, such as absolute continuity of the functions involved.

\subsection{The Dimensionality-Aware Outlierness Criterion.}

We now make use of the formulation in the statement of Theorem~\ref{T:ALDR} to produce a practical estimate of $\ALDR'$ on finite datasets. For this, we consider the value of $\ALDR'$ for small (but positive) choices of the limit parameters $\epsilon$ and $p$. 

Following the convention of $\LOF$, $\SLOF$, and other traditional outlier detection algorithms, the ball radius can be set to the familiar $k$-nearest neighbor distance,
$\epsilon=\kdist({\qq})$. Similarly, if $n$ is the size of the dataset, choosing $p=\nicefrac{k}{n}$ would set $\delta_{\qq}(p)$ and $\delta_{\oo}(p)$ to the distances at which their associated neighborhoods would be expected to contain $k$ samples out of $n$; these distances can be approximated by the $k$-NN distances $\kdist({\qq})$ and $\kdist({\oo})$, respectively. Note that by fixing $k$ to some reasonably small value, we have the desirable effect that the probability $p$ tends to zero as the dataset size $n$ increases.

Given these choices for $\epsilon$ and $p$, the expectation in the formulation of $\ALDR'$ can be estimated by taking the average over the $k$ nearest neighbors of $\qq$.

%Finally, we note that many estimators for $\IDstar_{F_{\oo}}$ are known to exist~\cite{}; in principle, any of them could be used. Rather than specifying the estimator here, we leave this as an implementation choice.

Using these approximation choices, we now state our proposed dimensionality-aware outlierness criterion, $\DAO$:
\begin{equation}
\label{E:DAO}
\DAO_k(\qq)
\:\: \triangleq \:\:
\frac{1}{k}
\sum_{\oo\in \nn_k(\qq)}
\left(
\frac{\kdist(\qq)}{\kdist(\oo)}
\right)^{\widehat{\IDstar_{F_{\oo}}}}
    \, ,
\end{equation}
where the neighborhood size $k$ is a hyperparameter, and the LID estimator $\widehat{\IDstar_{F_{\oo}}}$
is left as an implementation choice.

Although $\DAO$ is theoretically justified as an estimator of $\ALDR'$ by Theorem~\ref{T:ALDR}, it also can serve as an estimator of $\ALDR$.
Setting $\epsilon=\delta_{\qq}(\nicefrac{k}{n})$,
we obtain
\begin{eqnarray*}
\ALDR(\qq)
& \approx &
\frac{1}{k}
\sum_{\oo\in \nn_k(\qq)}
\frac{F_{\oo}(\delta_{\qq}(\nicefrac{k}{n}))}{F_{\qq}(\delta_{\qq}(\nicefrac{k}{n}))}
\\
& = &
\frac{1}{k}
\sum_{\oo\in \nn_k(\qq)}
\frac{F_{\oo}(\delta_{\qq}(\nicefrac{k}{n}))}{F_{\oo}(\delta_{\oo}(\nicefrac{k}{n}))}
\, ,
\end{eqnarray*}
each term of which can be approximated using the limit equality stated in Theorem~\ref{T:ALDR}:
\begin{equation*}
\ALDR(\qq)
\:\:\approx\:\:
\frac{1}{k}
\sum_{\oo\in \nn_k(\qq)}
\left(
\frac{\delta_{\qq}(\nicefrac{k}{n})}{\delta_{\oo}(\nicefrac{k}{n})}
\right)^{\IDstar_{F_{\oo}}}
\:\:\approx\:\:
\DAO_k(\qq)
\, .
\end{equation*}

We conclude this section by noting that $\DAO$ is nearly identical to $\SLOF$, with the exception that the \mbox{$k$-NN} distance ratio of $\DAO$ has exponent equal to the LID of the neighbor $\oo$. In essence, $\SLOF$ makes the implicit (but theoretically unjustified) assumption that the underlying local intrinsic dimensionalities are equal to 1 at every neighbor. We also note that the dimensionality-aware $\DAO$ criterion has a computational cost similar to that of $\SLOF$ whenever a linear-time LID estimator is employed (such as MLE ~\cite{LevBic04,AmsalegCFGHKN18}, reusing the \mbox{$k$-NN} queries also required to compute the outlier scores).

\section{Evaluation}\label{sec:eval}

\subsection{Methods and Parameters} \label{subsec:eval_methods_par}

\subsubsection{Outlier detection algorithms.}
We compare $\DAO$ against its dimensionality-unaware counterpart $\SLOF$~\cite{SchZimKri14}, as well as $\LOF$~\cite{BreKriNgSan00} and $\knn$~\cite{AngPiz02,RamRasShi00}, the two models with best overall performance from the extensive comparative study in \cite{CamZimSanCametal16}. For each method, we vary its neighborhood size hyperparameter $k$ from 5 to 100, and report the results for the value achieving the highest ROC AUC score. 

\subsubsection{LID estimators.}
%Since LID values used by $\DAO$ are not known in practice, we must first estimate them.
In our experimentation, we employ four different estimators of local intrinsic dimensionality: the classical maximum-likelihood estimator (MLE)~\cite{Hil75, LevBic04,AmsalegCFGHKN18}, tight local estimation using pairwise distances (TLE)~\cite{amsaleg2019, amsaleg2022intrinsic}, two-nearest-neighbor point estimation (TwoNN)~\cite{facco2017}, and estimation derived from changes in parametric probability density after data perturbation (LIDL)~\cite{tempczyk2022}. For MLE, TLE and TwoNN, we vary the neighborhood size hyperparameter settings across $\{5, 10, 15, 30, 50, 90, 150, 260, 320, 450, 560, 780\}$.
% The LIDL estimates the LID using the rate of change in the density estimate of the point when adding a series of normally distributed perturbations to the data.
For LIDL, we used 11 values geometrically distributed between 0.025 and 0.1 as the relative magnitude of the perturbation (the hyperparameter $\delta$), as well as \mbox{RQ-NSF} and MoG as density estimators~\cite{tempczyk2022}. All other aspects of the experimental design, including the neural network architecture and hyperparameter settings, followed the choices and recommendations made by the authors. 
% As density estimators, we employed Rational-Quadratic Neural Spline Flows (RQ-NSF)~\cite{durkan2019} and Mixture of Gaussians (MoG).

\subsubsection{Implementation and code.}
For this study, the MLE and TLE estimators of LID, and all four of the outlier detection algorithms, were implemented by us in Python. Our TLE implementation is based on the most recent publicly available MATLAB implementation provided by the authors~\cite{amsaleg2022intrinsic}. For TwoNN, we use the implementation available in the Python library \emph{scikit-dimension}~\cite{BacMGTZ21}\footnote{\url{https://scikit-dimension.readthedocs.io/}}. For LIDL, we use the publicly available Python implementation provided by the authors~\cite{tempczyk2022}\footnote{\url{https://github.com/opium-sh/lidl/}}. All code used in our experiments is available in our repository at \url{https://github.com/homarques/DAO/}.

\subsubsection{Computing infrastructure.}
Our experiments were performed using
a Quad-Core Intel Core i7 2.7 GHz processor
with 16 GB of RAM.
%Our experiments were performed in a machine with 252 GB RAM and a processor AMD EPYC 7501 32-Core, 2.0 GHz.

\subsection{Datasets}
\subsubsection{Synthetic datasets.}
\label{subsec:datasets}
We use synthetic data to evaluate the behavior of the methods in a controlled environment that allows their performance to be assessed as the local intrinsic dimensionality is varied. 
We generated datasets consisting of two clusters ($c_1$ and $c_2$) embedded in $\mathbb{R}^{32}$, with each cluster containing 800 data points drawn from a standard Gaussian distribution ($\mu = \mathbf{0}$, $\Sigma = \mathbf{I}$). 
To achieve a contrast in the LID values between the clusters, luster $c_1$ was generated within a subspace of dimension 8, and cluster $c_2$ within subspaces of dimensionality varying between 2 and 32. Initially, these subspaces were identified through a random selection of coordinate values, with all unused coordinates being set to zero. 

To generate a labeling of data points as `inlier' or `outlier', we computed the Mahalanobis distance to both cluster centers using their respective covariance matrices.
For a normally distributed cluster in a manifold of dimensionality $m$, the Mahalanobis distances from the points to the cluster center follow a $\chi^2$ distribution with $m$ degrees of freedom.
The points that exhibit a Mahalanobis distance to their cluster center larger than the 0.95 quantile were labeled as outliers, which results in an expected proportion of 5\% outliers per cluster.

Finally, after generating cluster members within these axis-parallel subspaces and labeling them as either `inlier' or `outlier', the resulting point clouds were then moved to more general positions, through a two-step procedure: translation by a vector with coordinates randomly drawn from a uniform distribution in $[-10, 10]$, followed by a random rotation. The rotation was performed by computing an orthonormal basis of a matrix with entries uniformly distributed in $[-1,1]$, and then projecting the point cloud to this new representation.

Overlap between $c_1$ and $c_2$ could result in outliers generated from one cluster being identified as an inlier of the other. In order to maintain separation between the clusters, we discarded (and regenerated) any dataset having one or more points that are simultaneously `close' to both cluster centers. This rejection condition was enforced strictly, by determining Mahalanobis distances to the two cluster centers and testing whether both distances are smaller than their respective $0.99999$ quantiles.

The data generation procedure described above was performed 30 times, to produce a total of 480 synthetic datasets (30 realizations of 16 dataset templates).

\subsubsection{Real-world datasets.}
% For experiments simulating real-world applications,
In our experimentation, we make use of datasets drawn from 6 different repositories for outlier detection. 
Since the outlier models of interest here implicitly assume the presence of continuous, real-valued attributes, we excluded datasets for which no attribute spanned at least 20\% distinct values. Also, for simplicity and ease of comparison, % (e.g., $\LOF$ when computing density estimates)
all duplicate records were dropped from the datasets. 
%Also, for computational reasons, we considered only datasets with at most 50,000 examples. 

A summary of the dataset collections is given in Table~\ref{tab:datasets}.
The first collection comprises 15 real datasets taken from a publicly available outlier detection repository~\cite{CamZimSanCametal16}.
The selection of datasets followed the same criterion adopted in~\cite{marques2020}, which is based on quantitative measures of suitability of a dataset for outlier detection benchmarking.
We also used 3 additional datasets from the augmented collection in~\cite{marques2020}.
The third collection contains 11 real multi-dimensional datasets taken from a publicly available outlier detection repository of various data types~\cite{rayana2016}. 
The fourth collection consists of 3 datasets used in a comparative study of algorithms for unsupervised outlier detection~\cite{goldstein2016}.
The fifth collection comprises 11 datasets used in a meta-analysis of anomaly detection methods~\cite{emmott2016}.
Finally, the last collection consists of 350 datasets from a publicly available outlier detection repository containing datasets engineered to have certain fixed proportions of outliers~\cite{kandanaarachchi2020}.
We selected real-valued datasets having 2\% of the members labeled as outliers. It is worth noting that the removal of duplicate points, if present, altered the overall proportion of outliers in these datasets.
Overall, 393 real datasets were selected for our experimentation.

\begin{table}[tbp]
\centering
\caption{\label{tab:datasets}Summary of 393 real datasets, with ranges showing numbers of features, dataset sizes, and proportions of outliers.}
\resizebox{0.49\textwidth}{!}{
\begin{tabular}{@{}lcccc@{}}
\toprule
Repository & \multicolumn{1}{c}{Features}  & \multicolumn{1}{c}{Size}            & \multicolumn{1}{c}{Outliers}        & Datasets \\ \midrule
Campos \textit{et al.} \cite{CamZimSanCametal16} & {[}5,  \hfill  259{]} & {[}50,  \hfill  49534{]} & {[}3\%, \hfill  36\%{]} &  15 \\
Marques \textit{et al.} \cite{marques2020} & {[}10,  \hfill  649{]} & {[}100,  \hfill  910{]} & {[}1\%, \hfill  10\%{]} &  3 \\
Rayana \cite{rayana2016} & {[}6,  \hfill  274{]} & {[}129,  \hfill  7848{]} & {[}2\%, \hfill  36\%{]} &  11 \\
Goldstein \& Uchida \cite{goldstein2016} & {[}27,  \hfill  400{]} & {[}367,  \hfill  49534{]} & {[}2\%, \hfill  3\%{]} &  3 \\
Emmott \textit{et al.} \cite{emmott2016} & [7, \hfill 128] & [992, \hfill 515129] & [9\%, \hfill 50\%] & 11 \\ 
Kandanaarachchi \textit{et al.} \cite{kandanaarachchi2020} & {[}2, \hfill  649{]} & {[}72,  \hfill  9083{]} & {[}1\%, \hfill  3\%{]} &  350 \\
\midrule
Overall & [2, \hfill 649] & [50, \hfill 515129] & [1\%, \hfill 50\%] & 393 \\ \bottomrule
\end{tabular}}
\end{table}

\subsubsection{Data availability.} 

For the sake of reproducibility, we have made available the collection of synthetic datasets used in our experiments, as well as the code used to generate them. For the real-world datasets, which come from third-party sources, we provide a list with the names of the dataset variants used in our experiments as well as pointers to the corresponding repositories. This information is available at \url{https://github.com/homarques/DAO/}.

\section{Experimental Results}\label{sec:results}

\subsection{Evaluation of LID Estimation on $\DAO$ Performance.} \label{sec:exp:synth:estimators}

We begin our experiments by first evaluating the performance of $\DAO$ when equipped with the LID estimates produced by 5 different estimators, namely, MLE, TLE, TwoNN, and 2 LIDL variants with different density estimators.
Figure \ref{fig:synthetic} shows the ROC AUC performance of all four outlier detection algorithms on the 480 synthetic datasets. The $x$-axis refers to cluster $c_2$, whose intrinsic dimension varies across datasets. Central points on solid lines indicate the average across 30 datasets, whereas the vertical bars represent the standard deviation. 

We first focus on the performance of $\DAO$ when equipped with different LID estimators. The use of MLE and TLE yielded the best performances. With these estimators, $\DAO$ showed no apparent loss of performance as the contrast in the intrinsic dimensions of the two clusters increased.
%These two estimators also showed the best runtime. On average, MLE took 0.03s to make the LID estimations for all the 12 different values of neighborhood size, while TLE took 56.5s.
Note, however, that TLE produces a smoothed LID estimate from pairwise distances within neighborhoods~\cite{amsaleg2022intrinsic}, which has the potential of compromising estimates for outlier points when all their neighbors are inliers. 
%Moreover, when the neighborhood size is large, TLE's use of pairwise distances makes it considerably more expensive to compute than MLE, which computes only one distance per neighbor.
Therefore, despite the good performance shown in our experiments, TLE is less preferable than MLE for outlier detection tasks.

TwoNN produces rough LID estimates using only the distances to the two closest neighbors. When using this simple estimator, $\DAO$ showed some loss of performance as the difference in the cluster dimensionalities increased. However, the drop in performance is less than that of both $\SLOF$ and $\LOF$, which do not take dimensionality into account. 
% Therefore, even by using simple estimates, one can see that it is beneficial to take into account the dimensionality of the ambient space.

The only LID estimator for which $\DAO$ (partially) underperformed $\SLOF$ and $\LOF$ is LIDL. This estimator relies on density estimation in the presence of normally-distributed perturbations. % to circumvent the high dimensionality challenges faced by traditional distance-based estimators. We used two different density estimation methods. A Mixture of Gaussians (MoG) and neural density estimator RQ-NSF. By using MoG for the density estimations, LIDL produced competitive results to those of $\SLOF$ up to 12 dimensions, even superior up to 8 dimensions. 
With Mixture of Gaussians (MoG) density estimation, Figure~\ref{fig:synthetic} shows a deterioration in the performance of $\DAO_{\LIDL}$ outlier detection as the intrinsic dimension increases, likely due to degradation in the quality of MoG density estimates. 
%
% In order to overcome the problem with density estimates in high-dimensional spaces, the authors recommended the use of neural density estimators.
When using instead the more sophisticated neural density estimator (RQ-NSF), the performance of $\DAO_{\LIDL}$ drops further, even on low-dimensional datasets. We conjecture that this is due to the fact that neural networks are biased to learn from the majority (inliers) and, as such, they may overlook the contributions of LID estimates in the vicinity of our targets (outliers).

\begin{figure}[tbp]\centering
  \includegraphics[width=\columnwidth]{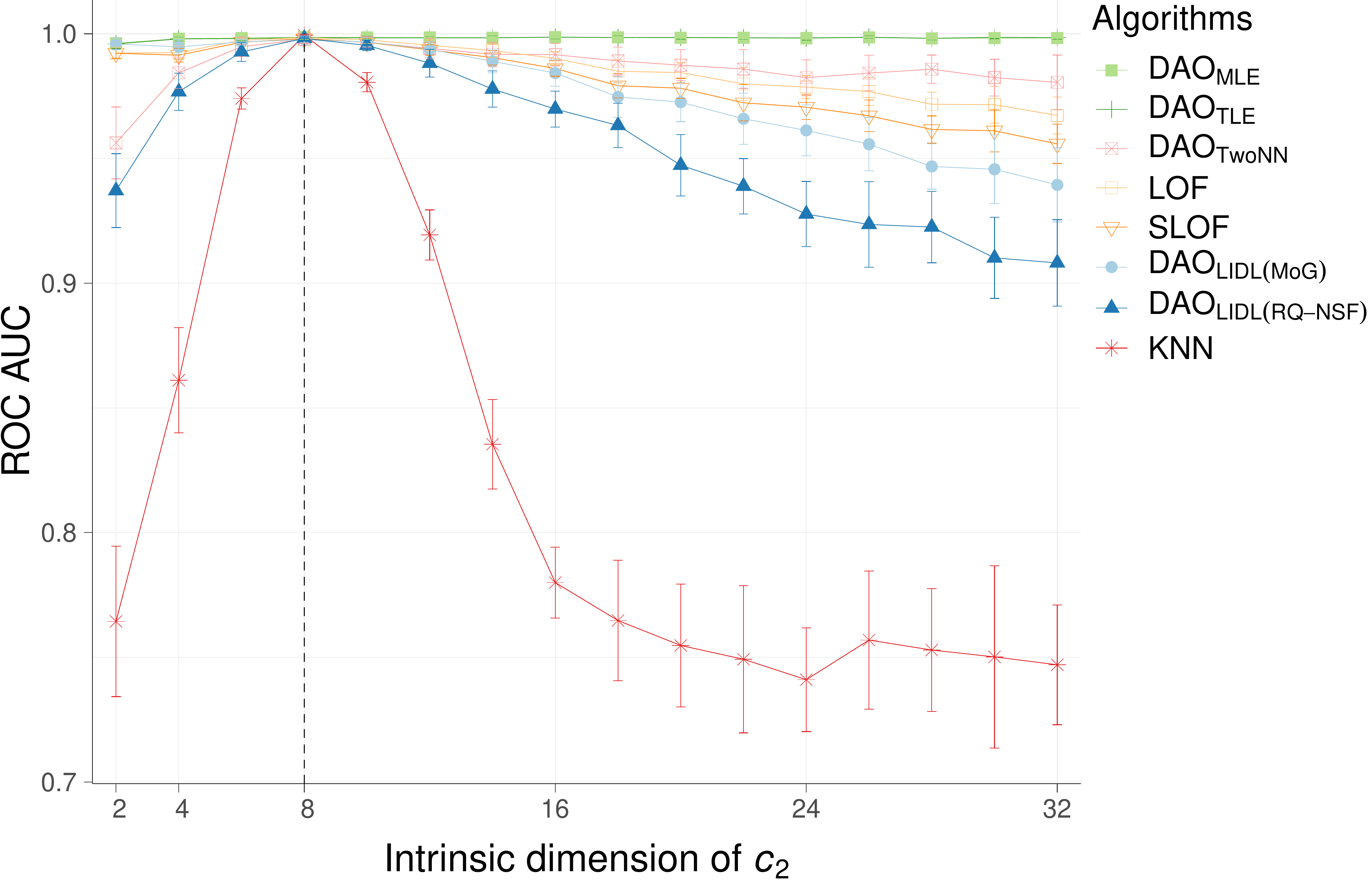}
  \caption{ROC AUC values for outlier detection performance over 480 synthetic datasets containing 2 clusters. One of the clusters ($c_1$) has intrinsic dimension fixed at 8. The intrinsic dimension of the other cluster ($c_2$) varies across the datasets ($x$-axis). The dashed vertical line indicates the reference set with both clusters sharing the same intrinsic dimension (8). The results shown are averages over 30 datasets with the same characteristics. Bars indicate standard deviation.}
  \label{fig:synthetic}
\end{figure}

% \subsubsection*{Final remarks} In this section, we thoroughly discussed the use of different estimators. 
%Overall, our experimental results showed to be beneficial to take into account the dimensionality of the space, even by using simple estimators such as TwoNN. However, one should be careful when making the LID estimates. Differently from other contexts, here, we are not only interested in the overall quality of the LID estimates but also in the quality of the estimates for specific regions where we may have outliers. Therefore, it is important to be cautious when using techniques such as the one used by TLE to make the LID estimates vary smoothly over the domain of the underlying distribution, or the use of LIDL, which turns out to learn from the mainstream by using neural networks. These techniques may cause undesired effects on the estimates in the regions of outliers. Nevertheless, they are still valuable tools for improving the overall quality of estimates in contexts other than outlier detection. Based on this discussion, in the following experiments, we use $\DAO$ with the estimates provided by MLE. However, except for LIDL, using any of the other estimators would lead to similar conclusions.

%\subsubsection*{Final remarks}
Therefore, when choosing an estimator of LID,
%may have been shown to be a valuable tool for improving the overall quality of estimates in contexts other than outlier detection, 
it is important to consider its underlying principles and mechanisms, and to assess whether it may adversely affect the performance of outlier detection. Based on the above discussion, in the following experiments involving $\DAO$, we choose MLE for the estimation of LID. 
%except for LIDL, using any of the other estimators would lead to similar conclusions.

\subsection{Comparative Evaluation on Synthetic Datasets.}
%In this section, we evaluated the ROC AUC performance of the dimensionally aware method $\DAO$ and 3 traditional non-dimensionally aware methods, namely, $\knn$, $\LOF$, and $\SLOF$. 
%We use the same collection of 480 synthetic datasets with two manifolds to assess the behavior of their ROC AUC as the difference in the dimensionality of the manifolds increases. We compared the performance of both categories of methods in this scenario to discuss the implications of neglecting the dimensionality of the space.

We begin our analysis with the synthetic dataset collection, focusing on the relative performance between $\DAO$ and its 3 dimensionality-unaware competitors. 
From Figure \ref{fig:synthetic}, % we can see the behavior of the ROC AUC performance for the different outlier detection methods as the difference in the dimensionality of the two manifolds increases.
one can see that when both clusters share the same intrinsic dimensionality (8), $\DAO$ and its dimensionality-unaware competitors perform equally well. However, as the difference in the dimensionality of the cluster manifolds increases, the performances of $\SLOF$, $\LOF$, and $\knn$ degrade noticeably. The experiments also show that of the various LID-aware variants considered, $\DAO_{\MLE}$ had consistently superior performance as the dimensionality of cluster $c_2$ was varied.

Table~\ref{tb:synth:lm} shows linear regression models fitted to predict the difference in ROC AUC between $\DAO_{\MLE}$ and each dimension-unaware method, as a function of the difference in the intrinsic dimension of the two data clusters manifolds. Among these methods, the greatest degradation of performance is that of $\knn$.
From the slope of the linear regression, one can see that on average, the ROC AUC performance of $\knn$ as compared to $\DAO$ decreased by almost 0.01 for each unit increment in the difference between the dimensions of the two clusters. 
%%%%%%In some cases, $\knn$'s ROC AUC dropped below 0.75, a difference of over 0.25 as compared with $\DAO$.
%%%%%% Note, however, that the rate of loss is not completely linear. Once $\knn$ deteriorated enough to reach ROC AUC around 0.78, the performance decreased at a much smaller rate, which resulted in a not-perfect Pearson correlation $\rho$ of 0.806. Therefore, the decrease in the ROC AUC performance of $\knn$ is expected to be even higher than 0.01 for slight variations in the dimensionality within the dataset.

%%%%%% The highest performance loss of $\knn$ may be explained by the fact that it not only completely ignores the dimensionality of the space, but also ignores the observations in the vicinity of the point.
%%%%%% $\SLOF$, which also ignores the dimensionality of the ambient space, takes into account the observations in the vicinity of the point by comparing the density estimate of the point with those of its neighbors. This approach results in a loss of performance at a much smaller rate compared to $\knn$.
%%%%% $\LOF$, which operates similarly to $\SLOF$, uses the reachability distance to bring additional smoothness into the density estimates. The use of the reachability distance seems to alleviate even further the loss of ROC AUC performance with the increase in the difference in the dimensionality of the manifolds. The rate of performance loss of $\SLOF$ is around 36\% higher compared to $\LOF$. 

These experimental outcomes on synthetic data confirm the theoretical analysis in Section~\ref{sec:dao}, in that the performance of $\SLOF$ is seen to degrade relative to its dimensionality-aware variant $\DAO$, as the differences in the dimensions of the cluster subspaces increase. As one might expect, the degradation of $\LOF$ is slightly less rapid than that of its close variant $\SLOF$; however, the performance drop for $\knn$ is much more drastic. One reason is that $\knn$ is a \emph{global} outlier detection method with scores expressed in the units of distance, as opposed to $\LOF$ and $\SLOF$, which are unitless \emph{local} methods based on density ratios. $\knn$'s use of absolute distance thresholds as the outlier criterion tends to favor the identification of points from higher-dimensional local distributions as outliers, due to the concentration effect associated with the so-called `curse of dimensionality'. This tendency becomes more pronounced as the relative difference between the underlying dimensionalities increases.
%%%%%%Since the dimensionality was kept constant \emph{within each cluster}, the local relative densities used by $\LOF$ and $\SLOF$ and, accordingly, their resulting outlier scores are, within each cluster and in relative terms, less impacted by neglecting ID information. It is the difference across clusters that make the inter-cluster outlier scores less comparable in these methods, causing their ranks to swap and performance to degrade, despite their local computations. In practice, though, datasets may exhibit much more complex dimensionality profiles, with LIDs varying noticeably even within neighbourhoods. That is why in Section \ref{subsec:realdata} we will significantly extend our experimental analysis by using a large and diverse suite of real datasets.     

%%%%%% As for $\DAO$, which takes into account not only the observations in the vicinity of the point but also the dimensionality of the space, the performance of ROC AUC remains stable with the increase in the difference in the dimensionality of the manifolds.

\begin{table}[tbp]
\caption{Simple linear regression to predict the difference in ROC AUC between $\DAO_{\MLE}$ and its dimensionality-unaware competitors on synthetic datasets. The explanatory variable is the absolute difference between the intrinsic dimensions of the two cluster manifolds. For each, we show the slope $m$, the $p$-value, and the Pearson correlation $\rho$.\label{tb:synth:lm}}
\centering
\setlength{\tabcolsep}{4pt}
\begin{tabular}{l|ccc}
\toprule
ROC AUC             & \multicolumn{3}{c}{\specialcell{Regression on the absolute difference \\ between the IDs of the manifolds}} \\ \midrule
                    & $m$                & $p$              & $\rho$                \\
$\DAO:\knn$    & 0.0099               &  1e-4                  &  0.806                 \\
$\DAO:\SLOF$   & 0.0018               &  8e-14                 &  0.991                \\
$\DAO:\LOF$    & 0.0013               &  3e-14                 &  0.992                \\ \bottomrule
\end{tabular}
\end{table}

\begin{table}[tbp]
\caption{Simple linear regression to predict the difference in ROC AUC between $\DAO_{\MLE}$ and its dimensionality-unaware competitors on real datasets. The explanatory variables are the dispersion $R$ and the Moran's I autocorrelation, both with respect to log-LID values. For each, we show the slope $m$, the $p$-value, and the Pearson correlation $\rho$. \label{tb:lm}}
%\small
\setlength{\tabcolsep}{4pt}
\begin{tabular}{@{}l|ccc|ccc}
\toprule
\multicolumn{1}{l|}{ROC AUC}     & \multicolumn{3}{c|}{$R$ (MAD)} & \multicolumn{3}{c}{Moran's I} \\ \midrule
             & $m$  & $p$  & $\rho$ & $m$     & $p$    & $\rho$    \\
$\DAO:\knn$   &  0.059  &  6e-3  &  0.14  &  -0.075  & 1e-5  &  -0.21 \\
$\DAO:\SLOF$   &  0.051  &  4e-12  &  0.34  &  -0.021  & 5e-4  &  -0.17 \\
$\DAO:\LOF$   &  0.046  &  5e-6  &  0.23  &  -0.016  & 5e-2  &  -0.1 \\
%$\DAO$ - Oracle   &  0.034  &  0.0009  &  0.17  &  -0.035  & 2-e5  &  -0.21      \\ 
\bottomrule
\end{tabular}
\end{table}

%\begin{figure*}[tbp]\centering
\begin{figure*}[t]\centering
\subfigure[$\DAO$ vs.~$\SLOF$]{\includegraphics[width=0.35\textwidth]{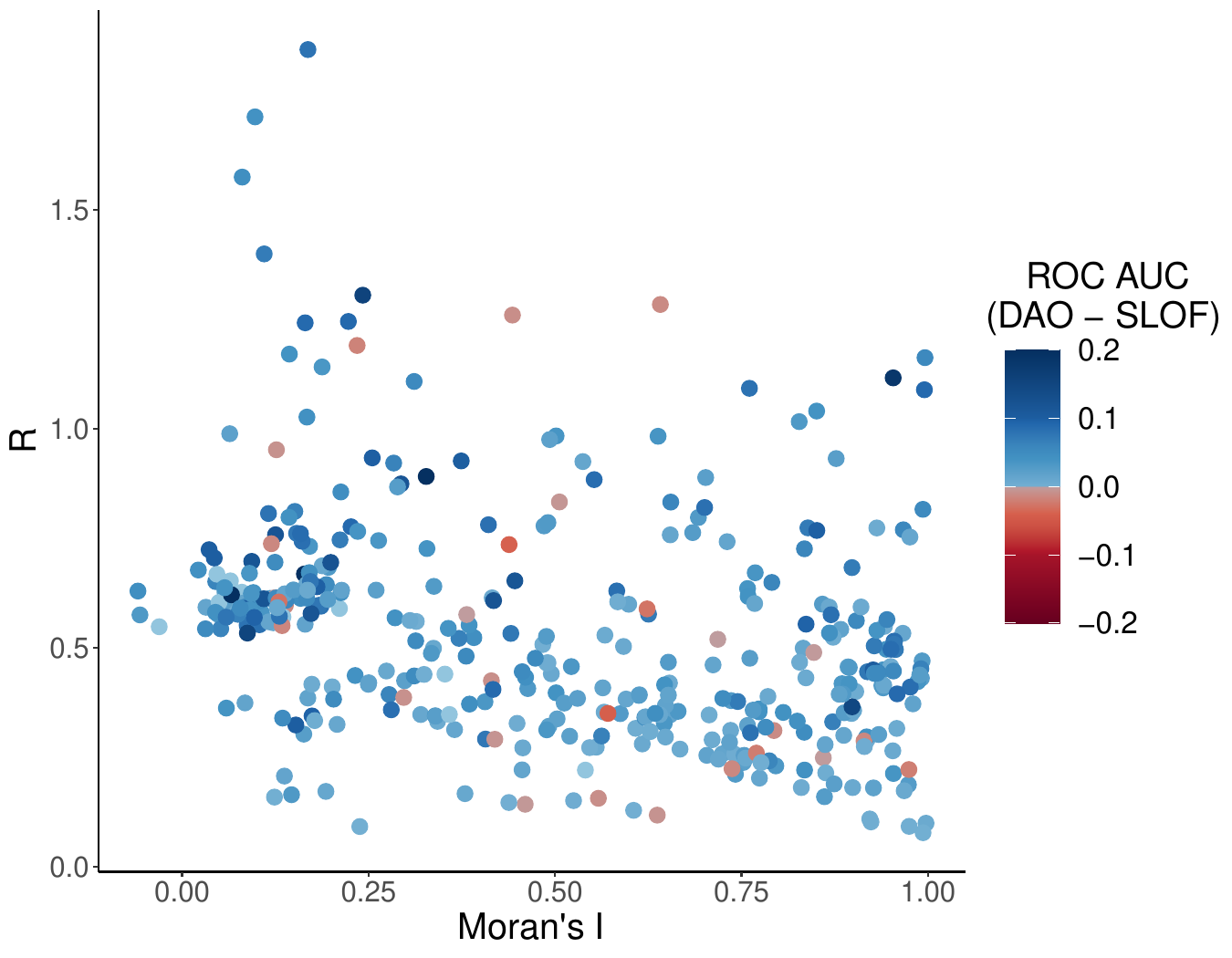}\label{fig:scatterplot:slof}}
\subfigure[$\DAO$ vs.~$\LOF$]{\includegraphics[width=0.35\textwidth]{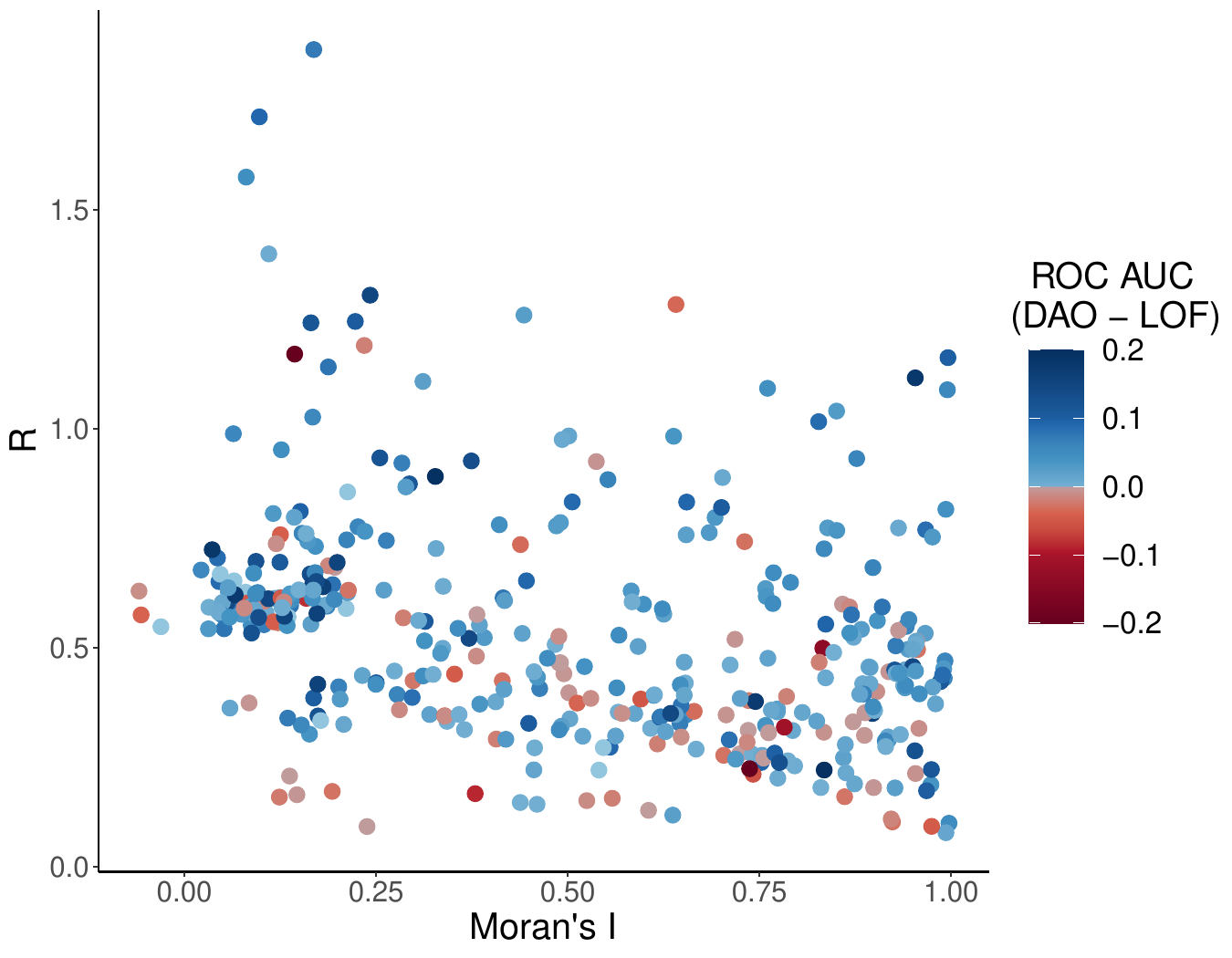}\label{fig:scatterplot:lof}}
\subfigure[$\DAO$ vs.~$\knn$]{\includegraphics[width=0.35\textwidth]{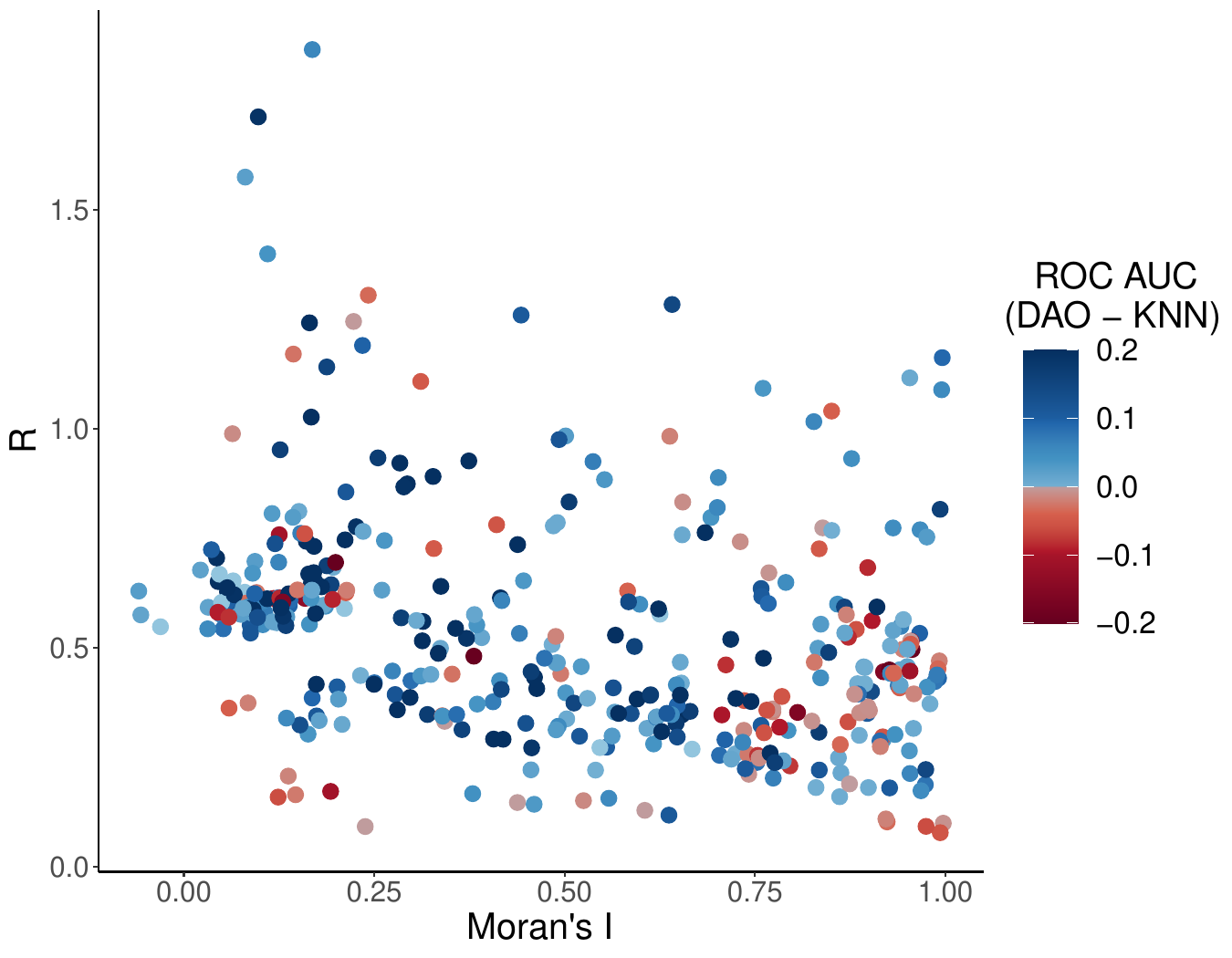}\label{fig:scatterplot:knn}}
\subfigure[$\DAO$ vs.~Oracle]{\includegraphics[width=0.35\textwidth]{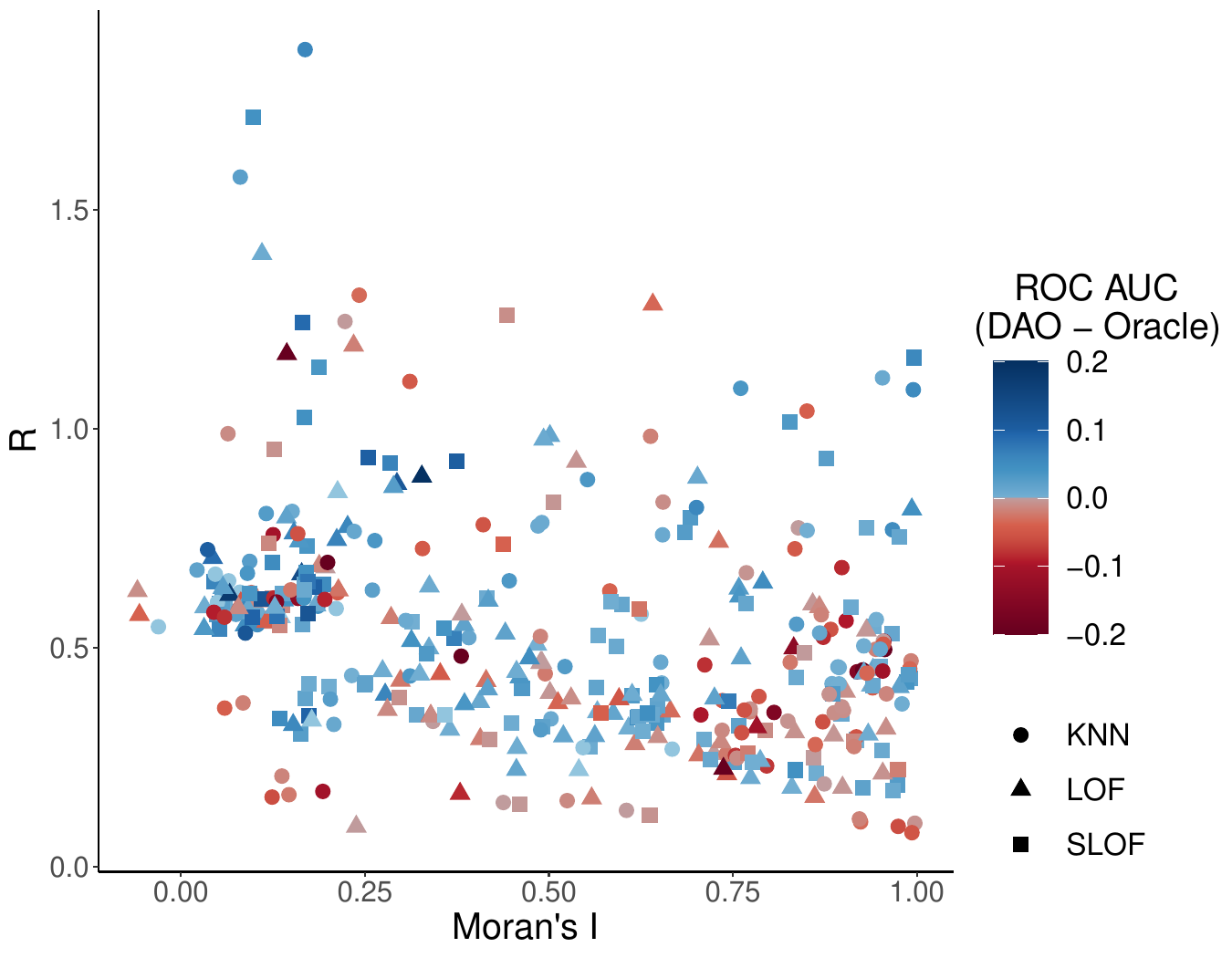}\label{fig:scatterplot:oracle}}
%\subfigure[$\DAO$ vs.~$\SLOF$]{\includegraphics[width=0.35\textwidth]{figs/slof.pdf}\label{fig:scatterplot:slof}}
%\subfigure[$\DAO$ vs.~$\LOF$]{\includegraphics[width=0.35\textwidth]{figs/lof.pdf}\label{fig:scatterplot:lof}}
%\subfigure[$\DAO$ vs.~$\knn$]{\includegraphics[width=0.35\textwidth]{figs/knn.pdf}\label{fig:scatterplot:knn}}
%\subfigure[$\DAO$ vs.~Oracle]{\includegraphics[width=0.35\textwidth]{figs/oracle.pdf}\label{fig:scatterplot:oracle}}
\caption{Differences in ROC AUC performance between $\DAO_{\MLE}$ and the dimensionality-unaware methods over 393 real datasets. Blue dots indicate datasets where $\DAO$ outperforms its competitor, whereas red dots indicate the opposite. The `Oracle' method indicates the best-performing competitor for each individual dataset. Color intensity is proportional to the ROC AUC difference. On the $x$- and $y$-axis we show the Moran's I autocorrelation and dispersion $R$ (mean absolute difference) of log-LID estimates, respectively.}
\label{fig:scatterplot}
\end{figure*}

\subsection{Comparative Evaluation on Real Datasets} \label{subsec:realdata}
% In the previous section, we compared the non-dimensionally aware methods against $\DAO$ in a large collection of synthetic datasets. We showed the effectiveness of $\DAO$ in handling datasets with large variability in the dimensionality within the dataset. In this section, we show that those same scenarios simulated with the synthetic datasets match real-world problems. In order to do that, we compare $\DAO$' performance against the same set of competitors in a large collection of 393 real datasets. 

\subsubsection{Dispersion of LID.}

We compared the performance of our dimensionality-aware outlier detection algorithm $\DAO_{\MLE}$ versus its dimensionality-unaware competitors over the collection of 393 real datasets described in Section~\ref{subsec:datasets} (see Table~\ref{tab:datasets}).
% As shown in the experiments on synthetic datasets, as the difference in the dimensionality within the dataset increases, the non-dimensionally aware methods tend to lose performance. However, for real datasets, it is not known a priori how the dimensionality varies within the dataset.
For real datasets, we cannot control how the local dimensionality varies within a dataset. In order to estimate this variability, we measure the dispersion of the LID estimates within each dataset using the mean absolute difference, as follows:
\begin{equation}
\label{eq:ratio}
%R = \frac{1}{n(n-1)}\sum_{i, j = 1}^{n} 
R = \frac{2}{n(n-1)}\sum_{1\leq i<j \leq n}
\left|
%\ln\left(\frac{\LID_i}{\LID_j}\right)
\ln\IDesti{i} - \ln\IDesti{j}
\right|
\, ,
\end{equation}
where $\IDesti{t}$ is the LID estimate for the $t$-th data point, as computed for $\DAO_{\MLE}$. 

This formulation in terms of log-LID values focuses on the difference in scale of the local intrinsic dimensionalities rather than the absolute differences themselves, which has the advantage that an incrementation of the dimensionality is treated as more significant for low dimensions (such as from 1 to 2) than it would be for high dimensions (such as from 24 to 25). We choose the $L_1$ formulation (mean absolute difference) instead of $L_2$ (variance) so as to avoid the dispersion score giving a disproportionately greater weight to the most extreme differences in log-LID.

% With this measure, we can capture the variance in terms of LID within the dataset. As with the synthetic datasets, we expect that by not taking the dimensionality into account, traditional methods should have a more severe performance loss on datasets with a large LID variation, i.e., datasets with higher values of $R$. 

\subsubsection{Autocorrelation of LID.}

Although dispersion captures the overall variability of the LID profiles within a dataset, it alone does not account for their full complexity.
% Another important aspect of the variation of the dimensionality within the dataset is how this variation occurs. For example, we expect the impact of a variation across the manifold, such as in the synthetic datasets, to be less severe than those where there are overlapping between manifolds of different dimensionality.
If variability occurs locally, possibly (but not necessarily) in a region of overlap between different manifolds, the LID profiles tend to be more complex than when variability occurs only across well-behaved, non-overlapping manifolds, or smoothly within a single, large manifold.  
To characterize this aspect of LID variation, we use the global Moran's I spatial autocorrelation~\cite{moran1950}, with log-LID as the base statistic of interest.
%(where the use of log has already been justified above w.r.t. the dispersion $R$).
% With this statistic, we can measure whether, overall, the observations are located in a region of uniform dimensionality or whether there is a high contrast between the dimensionality where the observations are located and the dimensionality where their neighbors are. 
Moran's I measures correlation of the base statistic among neighboring locations in space. In the context of datasets consisting of multidimensional feature vectors, one can define the Moran's I spatial neighborhoods to be the nearest neighbor sets of each data point, as induced by the distance measure of interest for the task at hand.
% As the global Moran's I spatial auto-correlation measure the overall correlation between the observations and their neighbors, a value of the neighborhood size must be provided. In this case, we vary the neighborhood size from 5 to 100 to maximize the absolute value of this statistic.
%Since this requires that a neighborhood size be provided,
For our experimentation,
we vary the neighborhood size from 5 to 100, and use the size that maximizes the % absolute value of the Moran's I statistic.
magnitude of the spatial autocorrelation (regardless of sign).

\subsubsection{Visualizing Outlier Detection Performance.}

In Figure~\ref{fig:scatterplot}, for each of the 393 real datasets, we visualize the differences in ROC~AUC performance between $\DAO_{\MLE}$ and the dimensionality-unaware outlier detection methods. Each colored dot in the scatterplot represents a single dataset, where blue indicates the outperformance of $\DAO$ relative to its competitor, and red indicates underperformance. The $y$-axis indicates the dispersion $R$ (mean absolute difference) of log-LID values computed at the data samples, and the $x$-axis shows their Moran's I autocorrelation~\cite{Ans95}.

In Figure~\ref{fig:scatterplot:slof}, we compare the performance of $\DAO$ against $\SLOF$. As discussed previously, $\SLOF$ can be seen as a dimensionally-unaware variant of $\DAO$, which implicitly assumes that the local intrinsic dimensionality of the test point always equals 1. From the clear predominance of blue dots, one can see that ignoring the intrinsic dimension leads to a performance loss in most cases. 
% Furthermore, the intensity of the colors shows that most of the dark blue dots are located in the cluster at the top-left corner, i.e., datasets with high LID variance and low correlation between the LID where the observations are located and the LID where their neighbors are. In fact,
When fitting linear regression to predict the difference in ROC~AUC between $\DAO$ and $\SLOF$ (Table~\ref{tb:lm}), %one can see that the slope of 0.051 and Pearson correlation $\rho = 0.34$ 
the dispersion $R$ and the gain of performance of $\DAO$ relative to $\SLOF$ are seen to have a direct relationship, as indicated by the positive regression slope and Pearson correlation.
%; that is, as the overall variability of LIDs within a dataset increases, $\DAO$ tends to perform better than $\SLOF$. The slope of -0.021 and $\rho = -0.17$
On the other hand, an inverse relationship exists between the Moran's I autocorrelation and the performance of $\DAO$ relative to $\SLOF$, as seen from the negative regression slope and Pearson correlation. 
In other words, as the correlation decreases between the intrinsic dimension at a query location and those of its neighbors' locations, $\DAO$ tends to outperform $\SLOF$ by a greater margin. 
%In a nutshell, $\DAO$ tends to perform relatively better than $\SLOF$ as the complexity of the LID profiles in a dataset increases.

% \begin{table}[tbp]
% \caption{Summary of the simple linear regression to predict the difference in ROC AUC between the dimensionally aware method $\DAO$ and the non-dimensionally aware methods on real datasets. The table also displays the Pearson correlation ($\rho$) between the response and explanatory variables. \label{tb:lm}}
% \begin{tabular}{@{}lcccccc@{}}
% \toprule
% \multicolumn{1}{c}{ROC AUC}     & \multicolumn{3}{c}{Hellinger Distance} & \multicolumn{3}{c}{Moran's I for Normalized Entropy Power} \\ \midrule
%              & Slope  & $p$-value  & $\rho$ & Slope     & $p$-value    & $\rho$    \\
% $\DAO$ - $\SLOF$   &  0.072  &  2e-6  &  0.23  &  -0.018  & 3e-3  &  -0.15 \\
% $\DAO$ - $\LOF$   &  0.076  &  5e-4  &  0.17  &  -0.02  & 2e-2  &  -0.11 \\
% $\DAO$ - $\knn$   &  0.081  & 6e-2  &  0.09  &  -0.069  & 1e-4  &  -0.19 \\
% %$\DAO$ - Oracle   &  0.034  &  0.0009  &  0.17  &  -0.035  & 2-e5  &  -0.21      \\ 
% \bottomrule
% \end{tabular}
% \end{table}

Overall, the results and major trends are similar when comparing $\DAO$ 
against $\LOF$ in Figure~\ref{fig:scatterplot:lof}, 
against $\knn$ in Figure~\ref{fig:scatterplot:knn}, 
and even (to a lesser extent) against an oracle that uses the best-performing competitor for each individual dataset in Figure~\ref{fig:scatterplot:oracle}. Their respective regression analyses, shown in Table~\ref{tb:lm}, lead essentially to the same conclusions as for $\SLOF$. 
It is worth noting that $\knn$ exhibits the largest (absolute) regression coefficients, which is consistent with the results from the synthetic experiments.
%, where this algorithm showed the most susceptibility to high variability of LID values within a dataset. 

Our experimentation reveals that dimensionality-aware outlier detection is of greatest advantage when the dataset has a complex LID profile, as indicated by a high dispersion ($R$ value) and/or a low autocorrelation (Moran's I value). The four scatterplots of Figure~\ref{fig:scatterplot} all show that the dimensionality-unaware methods are more competitive when there is less contrast in the LID values across the dataset --- that is, when the dispersion is low or the autocorrelation is high. Note that no outlier detection method can be expected to have perfect performance, as there are multiple factors that can favor any given model over any other~\cite{CamZimSanCametal16}. For example, among the outlier models studied in this paper, $\knn$ is known to be favored when the dataset contains many distance-based outliers.

We also summarize the overall results in a critical distance diagram (Figure~\ref{fig:cd}), which shows the average ranks of the outlier detection methods with respect to ROC~AUC, taken across the 393 real datasets. The width of the upper bar (CD) indicates the critical distance of the well-known Friedman-Nemenyi statistical test at significance level $\alpha$ = 1e-16. The large gap between $\DAO$ and $\LOF$ serves as quantitative evidence that $\DAO$ outperformed its dimensionality-unaware competitors by a significant margin.

\begin{figure}[tbp]\centering
  \includegraphics[width=0.3\textwidth]{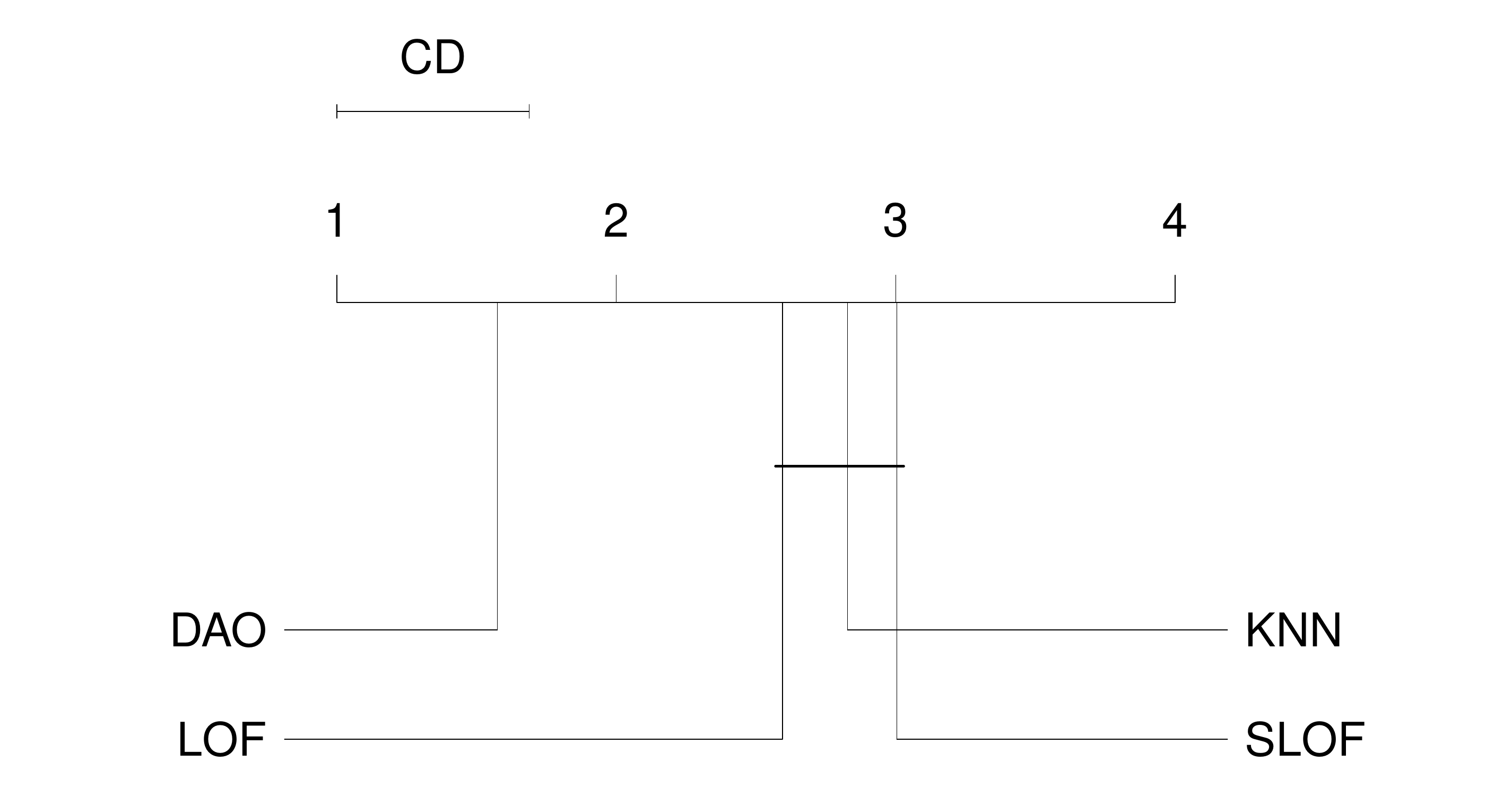}
  \caption{Critical difference diagram (significance level $\alpha$ = 1e-16) of average ranks of the methods on 393 real datasets: $\DAO_{\MLE}$ vs.\ baseline competitors.}
  \label{fig:cd}
\end{figure}

\subsection{Runtime Performance and Computational Complexity}

We evaluated the runtime performance of the outlier detection methods using our collection of 480 synthetic datasets with 1600 data points embedded in $\mathbb{R}^{32}$. 
%The execution times for all cases are dominated by the costs associated with the precomputation of $k$-nearest-neighbor sets for every point in the dataset. 
The three competing methods --- $\knn$, $\SLOF$, and $\LOF$ --- exhibited essentially the same execution time per run when averaged across all datasets and candidate neighborhood size values: 0.063s$\,\pm\,$0.008s, 0.064s$\,\pm\,$0.008s, and 0.065s$\,\pm\,$0.008s, respectively. 
% it took 0.103s$\,\pm\,$0.005s to run all 96 candidate values of neighborhood size $k$ per dataset.
% This best runtime performance, however, comes from the fact that it not only completely ignores the dimensionality of the space, but also ignores the observations in the vicinity of the point.
% $\SLOF$ also ignores the dimensionality of the ambient space, but it takes into account the observations in the vicinity of the point by comparing the density estimate of the point with those of its neighbors. This comparison results in an increase in the runtime.
% $\SLOF$ took, on average, 0.198s$\,\pm\,$0.005s to run all 96 values of $k$ per dataset.
%However, the average times per run for $\LOF$ and $\SLOF$ were 0.065s$\,\pm\,$0.008s and 0.064s$\,\pm\,$0.008s, respectively. Although all three employ the same routine to perform NN queries as $\knn$, the difference in runtime is explained by $\SLOF$'s computations of relative densities within neighbourhoods. These are also performed by $\LOF$, but in a slightly more sophisticated way, with the use of reachability distances for smoothness, which comes with some additional runtime.  
%$\LOF$ operates similarly to $\SLOF$. However, it uses the reachability distance to bring additional smoothness into the density estimates. The use of the reachability distance further increased the runtime performance of the method.
% On average, it took 0.286s$\,\pm\,$0.006s for $\LOF$ to run all 96 candidate values of $k$ per dataset.
%On average, $\LOF$ took 0.065s$\,\pm\,$0.008s per run across all datasets and candidate neighbourhood size values.
%
In contrast to the other methods, the $\DAO$
% criterion considers not only the observations in the vicinity of the point but also the dimensionality of the space. Making the method dimensionally aware, however, led to the worst runtime performance in our experiments.
criterion also required the estimation of LID values, which in our framework used neighborhoods of size up to 780 --- several times larger than the maximum neighborhood size (100) used by the competing methods. On average, the $\DAO$ execution time was 0.095s$\,\pm\,$0.036s per run across all datasets and candidate neighborhood size values. 

Asymptotically, the computational cost of all algorithms under consideration is dominated by that of determining neighborhood sets for all data points, which in the most straightforward (brute force) implementation requires $\Theta(n^2)$ distance calculations. With appropriate index structures, such as a K-d-Tree, subquadratic time may be achievable provided that the data dimensionality is not too high. 
%Notice that, in the case of $\DAO$, this analysis is not affected when considering any LID estimator that runs in sub-quadratic time w.r.t. the dataset size. This includes MLE, which runs in linear time once nearest neighbors have been precomputed. 
However, even with indexing support, the runtime complexity of $\DAO$ is essentially the same as that of $\knn$, $\SLOF$, and $\LOF$.

\section{Conclusion}\label{sec:conclusions}

% In contrast to supervised learning, where methods are supported by theoretical foundations~\cite{Vapnik95,Vapnik98}, methods in unsupervised learning are often supported only by empirical studies~\cite{CamZimSanCametal16,goldstein2016,emmott2016}.
%In this paper, we contributed to bridging the gap between theory and empirical studies in unsupervised outlier detection by presenting $\DAO$, a theoretical model for local outlier detection. 

In our derivation of $\DAO$ via the theoretical LID model, and its subsequent empirical validation, we have made the case for a dimensionality-aware treatment of the problem of outlier detection. The theoretical and empirical evidence presented in this paper establishes that conventional, dimensionality-unaware approaches are susceptible to the variations and correlations in intrinsic dimensionality observed in most real datasets,  and that the theory of local intrinsic dimensionality allows for a more principled treatment of outlierness.

% Using the tools available in the literature, we implemented our theoretical model $\DAO$ and compared it against the state-of-the-art for outlier detection. Our implementation of $\DAO$ outperformed the competitors with statistical significance on a large collection of over 390 real-world datasets. 
%We ran extensive experiments involving 480 synthetic and 393 real-world datasets, whereby we showed a statistically significant correlation between the decrease in ROC AUC performance of dimension-unaware methods and the increase in variability and in the complexity of local intrinsic dimensionality (LID) profiles within a dataset. By making our model dimension aware, its ROC AUC performance became significantly more robust in this regard.

Our analyses have shed some light on the fact that the quality of dimensionality-aware local outlier detection depends crucially on the properties of the estimator of LID. Estimators that learn by optimizing an objective function that favors inliers (such as LIDL), or those that perform smoothing (such as TLE), should be either avoided or used with caution. As our experiment results suggest, the use of an unsuitable estimator of LID may introduce errors that may outweight the benefits of dimensionality-aware techniques. It is still an open question as to which estimators of LIDs lead to the best outlier detection performance in practice. However, in our experimentation involving synthetic data, and the success of $\DAO_{\MLE}$ against top-performing nonparametric outlier methods ($\LOF$, $\SLOF$ and $\knn$) on hundreds of real datasets, we have seen the emergence of the MLE estimator of LID as a sensible option for practical outlier detection tasks.          

% Finally, it is important to remember that other implementations of $\DAO$ are possible depending on the choices to estimate the quantities from the finite sample.
% However, as discussed in our experiments (Section \ref{sec:exp:synth:estimators}), one should care not only about the overall quality of the estimates but also about the quality of the estimates for specific regions where we may have outliers. 
% Furthermore, the quality of $\DAO$ is directly connected to the quality of the LID estimates, and this estimation is not a simple task per se. Therefore, in addition to runtime overhead to compute the estimates, the inclusion of poor estimates may lead to additional estimation errors in the method, resulting in the superior performance of the non-dimensionally aware method (Figure \ref{fig:scatterplot:slof}).

\bibliographystyle{IEEEtranS}
\bibliography{abbrev,literature}

\end{document}